\documentclass[final,times,3p,twocolumn]{elsarticle}

\usepackage{graphicx}
\usepackage{amssymb}
\usepackage{amsmath}
\usepackage{amsthm}
\usepackage{fixltx2e}

\usepackage{caption}
\usepackage{subcaption}

\journal{Image and Vision Computing-IMAVIS-D-14-00231}

\providecommand{\mtHomogeneous}[1]{\tilde{#1}}
\providecommand{\mtEstimate}[1]{\breve{#1}}
\providecommand{\mtError}[1]{\bar{#1}}
\providecommand{\mtNormalize}[1]{\hat{#1}}

\providecommand{\mtVec}[1]{\boldsymbol{\mathbf{#1}}}
\providecommand{\mtHVec}[1]{\mtHomogeneous{\mtVec{#1}}}
\providecommand{\mtUnitVec}[1]{\mtNormalize{\mtVec{#1}}}
\providecommand{\mtEstVec}[1]{\mtEstimate{\mtVec{#1}}}
\providecommand{\mtErrVec}[1]{\mtError{\mtVec{#1}}}

\providecommand{\mtMat}[1]{\boldsymbol{\mathbf{#1}}}

\providecommand{\mtVar}[1]{#1}
\providecommand{\mtHVar}[1]{\mtHomogeneous{#1}}

\providecommand{\mtBy}[2]{ {#1} \times {#2} }

\providecommand{\mtPoint}[1]{\mtVec{p}^{#1}}
\providecommand{\mtHPoint}[1]{\mtHVec{p}^{#1}}

\providecommand{\mtRotation}[0]{\mtMat{\mathcal{R}}}
\providecommand{\mtR}[1]{\mtRotation_{#1}}
\providecommand{\mtRt}[1]{\mtR{#1}^{\top}}

\providecommand{\mtTranslation}[0]{\mtVec{t}}
\providecommand{\mtt}[1]{\mtTranslation_{#1}}

\providecommand{\mtTransformation}[0]{\mtMat{T}}
\providecommand{\mtT}[2]{\mtTransformation_{#1}^{#2}}

\providecommand{\mtTK}[0]{\mtTransformation_{\mtFrame{K}}}

\providecommand{\mtRK}[0]{\mtRotation_{\mtFrame{K}}}

\providecommand{\mtRKt}[0]{\mtRotation_{\mtFrame{K}}^{\top}}

\providecommand{\mttK}[0]{\mtTranslation_{\mtFrame{K}}}

\providecommand{\mtProj}[1]{\mtMat{\kappa}_{#1}}
\providecommand{\mtIdentity}[1]{\mtMat{I}_{\mtBy{#1}{#1}}}
\providecommand{\mtZeros}[2]{\mtMat{0}_{\mtBy{#1}{#2}}}
\providecommand{\mtSO}[1]{{SO({#1})}}
\providecommand{\mtSE}[1]{{SE({#1})}}

\providecommand{\mtFrame}[1]{#1}
\providecommand{\mtCam}[1]{\mtFrame{C}_{#1}}
\providecommand{\mtKey}[1]{\mtFrame{K}_{#1}}

\providecommand{\mtCiKj}[2]{ \mtCam{#1} \mtKey{#2} }
\providecommand{\mtCK}[2]{ {#1},{#2} } % Short-form for previous
\providecommand{\mtImage}[1]{\mtFrame{I}_{#1}}

\providecommand{\mtTarget}[0]{\mtFrame{M}}

\providecommand{\mtTc}[1]{\mtTransformation_{\mtCam{#1}}}

\providecommand{\mtTk}[1]{\mtTransformation_{\mtKey{#1}}}

\providecommand{\mtRk}[1]{\mtR{\mtKey{#1}}}
\providecommand{\mtRkt}[1]{\mtRt{\mtKey{#1}}}
\providecommand{\mtRc}[1]{\mtR{\mtCam{#1}}}
\providecommand{\mtRct}[1]{\mtRt{\mtCam{#1}}}
\providecommand{\mttk}[1]{\mtt{\mtKey{#1}}}
\providecommand{\mttc}[1]{\mtt{\mtCam{#1}}}

\providecommand{\mtSkewSym}[1]{\left[{#1}\right]_\times}
\providecommand{\mtTranspose}[1]{{{#1}^{\top}}}
\providecommand{\mtInverse}[1]{{#1}^{-1}}

\providecommand{\mtDet}[1]{\text{det} \left( #1 \right)}

\providecommand{\mtRank}[1]{\text{rank} \left( #1 \right)}
\providecommand{\mtReal}[1]{\mathbb{R}^{#1}}

\providecommand{\mtNatural}[0]{\mathbb{N}^{+}}
\providecommand{\mtProjective}[1]{\mathbb{P}^{#1}}

\providecommand{\mtPartial}[2]{ \dfrac{\partial {#1}}{\partial {#2}} }

\providecommand{\mtEvaluated}[2]{ \left. {#1} \right|_{#2} }
\providecommand{\mtCross}[2]{ {#1} \times {#2} }
\providecommand{\mtDot}[2]{ {#1} \cdot {#2} }

\providecommand{\mtEtAl}[1]{#1 \emph{et al}}
\providecommand{\mtIth}[1]{{#1}^{\text{th}}}

\providecommand{\mtArgMin}[1]{\underset{#1}{\text{arg\,min}}\ }

\providecommand{\mtJacobian}[0]{\mtMat{J}}
\providecommand{\mtJobs}[3]{\mtJacobian_{#1}^{\mtCK{#2}{#3}}}

\newtheorem{mtTheorem}{Theorem}[section]
\newtheorem{mtLemma}[mtTheorem]{Lemma}

\newtheorem{mtCorollary}[mtTheorem]{Corollary}
\theoremstyle{definition}

\newtheorem{mtAssumption}{Assumption}[section]

\begin{document}

\begin{frontmatter}

%% Title, authors and addresses

%% use the tnoteref command within \title for footnotes;
%% use the tnotetext command for the associated footnote;
%% use the fnref command within \author or \address for footnotes;
%% use the fntext command for the associated footnote;
%% use the corref command within \author for corresponding author footnotes;
%% use the cortext command for the associated footnote;
%% use the ead command for the email address,
%% and the form \ead[url] for the home page:
%%
%% \title{Title\tnoteref{label1}}
%% \tnotetext[label1]{}
%% \author{Name\corref{cor1}\fnref{label2}}
%% \ead{email address}
%% \ead[url]{home page}
%% \fntext[label2]{}
%% \cortext[cor1]{}
%% \address{Address\fnref{label3}}
%% \fntext[label3]{}

\title{Degenerate Motions in Multicamera Cluster SLAM with Non-overlapping Fields of View}

%% use optional labels to link authors explicitly to addresses:
%% \author[label1,label2]{<author name>}
%% \address[label1]{<address>}
%% \address[label2]{<address>}

\author[uwmme]{Michael J.~Tribou\corref{cor1}}
\ead{mjtribou@uwaterloo.ca}

\author[uwece]{David W.~L.~Wang}
\ead{dwang@uwaterloo.ca}

\author[uwmme]{Steven L.~Waslander}
\ead{stevenw@uwaterloo.ca}

\cortext[cor1]{Corresponding author at: University of Waterloo, E3X-4118 -- 200 University Avenue West, Waterloo, ON, Canada, N2L 3G1. Telephone: +1-519-635-8971}

\address[uwmme]{Department of Mechanical and Mechatronics Engineering, University of Waterloo, 200 University Avenue West, Waterloo, ON, Canada, N2L 3G1.}
\address[uwece]{Department of Electrical and Computer Engineering, University of Waterloo, 200 University Avenue West, Waterloo, ON, Canada, N2L 3G1.}

%% ABSTRACT
\begin{abstract}
An analysis of the relative motion and point feature model configurations leading to solution degeneracy is presented, for the case of a Simultaneous Localization and Mapping system using multicamera clusters with non-overlapping fields-of-view. The SLAM optimization system seeks to minimize image space reprojection error and is formulated for a cluster containing any number of component cameras, observing any number of point features over two keyframes. The measurement Jacobian is transformed to expose a reduced-dimension representation such that the degeneracy of the system can be determined by the rank of a dense submatrix. A set of relative motions sufficient for degeneracy are identified for certain cluster configurations, independent of target model geometry. Furthermore, it is shown that increasing the number of cameras within the cluster and observing features across different cameras over the two keyframes reduces the size of the degenerate motion sets significantly.
\end{abstract}

\begin{keyword}
SLAM \sep
Computer vision \sep
Multicamera cluster \sep
Non-overlapping FOV \sep
Degeneracy analysis \sep
Critical motions
%% keywords here, in the form: keyword \sep keyword
%% MSC codes here, in the form: \MSC code \sep code
%% or \MSC[2008] code \sep code (2000 is the default)
\end{keyword}

\end{frontmatter}

%%
%% Start line numbering here if you want
%%
% \linenumbers

%% main text
% INTRODUCTION
\section{Introduction}
\label{section:introduction}
Precise robotic motion and manipulation tasks with respect to unknown target environments and objects require an accurate, real-time measurement of the relative position and orientation of the robot and target. Multicamera systems are often employed for robotic pose and target model estimation, as each camera is an inexpensive, light-weight, and passive device capable of collecting a large amount of environment information at high rates. Many researchers across different fields have investigated the use of cameras for the purpose of estimating motion and scene structure. As a result, many techniques using a variety of camera types and configurations have been detailed in the literature.

A camera cluster is composed of any number of simple perspective cameras mounted rigidly with respect to each other, as shown in Figure \ref{fig:cluster}, including configurations in which their fields-of-view (FOV) are spatially disjoint \cite{Baker:2001:semcmbmw}. This arrangement makes effective use of the camera sensors to cover a large combined FOV with high resolution, and in general, is able to overcome the limitations of other camera configurations, such as scale and translation-rotation motion ambiguities \cite{Pless:2003:Umcao}. Additionally, by arranging the cameras to look in many directions, the pose estimation is made more robust since when certain cameras do not see any point features suitable for tracking, the other cameras in the cluster can maintain the localization. In this scenario, camera arrangements with a smaller collective FOV may become lost causing the tracking operation to fail.

\begin{figure}[!ht]
\centering
\includegraphics[width=0.48\textwidth]{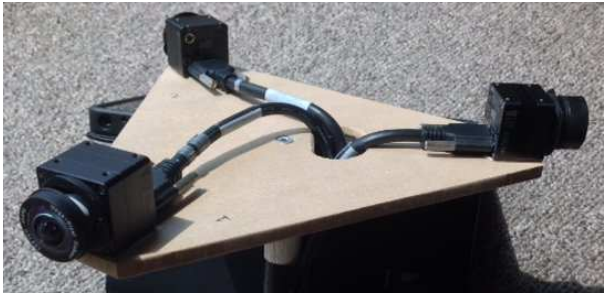}
\caption{An example camera cluster in which the three component cameras are rigidly-fixed with respect to each other.}
\label{fig:cluster}
\end{figure}

In order for any pose estimation system to operate successfully, the current state must be uniquely recoverable given the measurable outputs up to, and including the current time step. In the context of a multicamera cluster relative pose system, this means that the image measurements must contain sufficient information to recover the cluster motion and the target model parameters, including the proper global scale metric. Furthermore, the solution must be unique since convergence to a different configuration, which may also agree with the measurements, would likely result in failure of perception and control operations.

When the multicamera cluster is configured such that there is little or no spatial FOV overlap between the component cameras, the sensitivity of the image measurements to the global scale of the reconstructed model is low, particularly around specific motion profiles known as critical motions \cite{Clipp:2008:RDMENMS}. When the relative motion of the cluster is at or near critical, the global scale of the solution is extremely difficult, if not impossible, to recover accurately. In the presence of  measurement noise, the solution will converge to an incorrect scale value.

This work investigates the degenerate configurations when estimating the Simultaneous Localization and Mapping (SLAM) \cite{Thrun:2005:PR} system states for a calibrated multicamera cluster over two keyframes while observing a set of point features in each camera and using an iterative optimization or recursive filter-based approach, minimizing the image space reprojection error of point feature measurements. This includes Bundle Adjustment (BA) \cite{Hartley:2003:Mvgcv} schemes as well as recursive filters such as an extended Kalman filter \cite{Davison:2007:MRSCS}. The main contribution is the identification of configurations of motion and target model structure leading to non-unique SLAM solutions. %While other researchers have provided analyses of other solution methods \cite{Kim:2010:Remsmtc,Kim:2010:DotLSPAfGE,Clipp:2008:RDMENMS}, this is the first work to consider degeneracies for multicamera cluster SLAM systems employing BA or recursive filters solving for motion and structure concurrently. 

Determining the system configurations leading to solution degeneracy is closely related to the concept of observability in control systems. In the study of observability for nonlinear systems, the local weak observability of the system can be determined by calculating the observability rank condition about any point in the state space \cite{Hermann:1977:Nco}. This involves checking the column rank of a matrix containing the partial derivatives with respect to the system states, for increasing orders of Lie derivatives of the measurement model with respect the the system dynamics. When the matrix has full column rank, the system is locally weakly observable about that point. 

For a SLAM system using only the visual measurements from the cluster cameras and a non-stationary target, the system does not have a model of the dynamics for the relative motion and therefore, only the zeroth-order Lie derivatives are non-zero. In this case, evaluating the observability rank condition is equivalent to checking the rank of the measurement Jacobian matrix, as will be done here in the degeneracy analysis in Section \ref{section:degeneracies}. If the system were to contain a model of the relative motion dynamics, and the extra information that comes with it, the higher-order Lie derivatives of the measurement model would contain non-zero terms and the added matrix rows would only increase the likelihood that the matrix has full column rank at any point in the state space. However, in this analysis, no such assumptions about the relative motion dynamics are made and the degenerate configurations arising from only using image measurements for a set of point features over two keyframes are identified.

The remainder of this paper is arranged as follows: Section \ref{section:background} contains a review of the previous analyses for degenerate configurations of the multicamera cluster relative pose system; Section \ref{section:multicamera} presents the multicamera cluster SLAM system; the degenerate configurations of the pose estimation system are identified in Section \ref{section:degeneracies}; and finally, conclusions are drawn in Section \ref{section:conclusions}.

% BACKGROUND
\section{Related Work}
\label{section:background}

Previous analyses identifying cluster motions leading to degenerate system solutions have assumed that the five degrees of freedom describing relative orientation and translation direction of the cluster are known using the well-studied single camera ego-motion estimation techniques (e.g.~\cite{Hartley:2003:Mvgcv}). These include the work of \mtEtAl{Kim}.~\cite{Kim:2010:Remsmtc}, and \mtEtAl{Clipp}.~\cite{Clipp:2008:RDMENMS} for camera clusters with two component cameras, as well as that of the authors \cite{Tribou:2013:srmcsnfov} for clusters with three component cameras. Of interest are the conditions when the image measurements from the camera cluster are able to allow for estimation of the final degree of freedom, corresponding to the translation magnitude and therefore, global system scale. The analyses show that when each point feature is seen by only one of the two cameras at both keyframes, the global scale of the solution solution is recoverable only when the relative translational and rotational motion are both non-zero, and does not result in the optical centres of each camera moving in concentric arcs on circles with a common centre at the intersection of the baselines at each keyframe \cite{Clipp:2008:RDMENMS}. When a third non-collinear camera is added to the cluster, the set of degenerate motions is reduced to those which result in all the three cameras moving in parallel \cite{Tribou:2013:srmcsnfov}.

Analyses of degeneracies of the full SLAM solution for multicamera clusters have focused on those associated with solving the generalized camera relative pose problem, either linearly using the Generalized Essential Matrix (GEM) \cite{Pless:2003:Umcao}, or aligning imaging rays in space for minimal cases of camera poses and points \cite{Stewenius:2005:StMGRPP}. Sturm \cite{Sturm:2005:MGGCM}, \mtEtAl{Stewenius}.~\cite{Stewenius:2005:StMGRPP}, and \mtEtAl{Mouragnon}.~\cite{Mouragnon:2009:Grsmulba} discuss some degenerate cases, but Kim and Kanade \cite{Kim:2010:DotLSPAfGE} provide the most complete analysis. They identify the following degenerate configurations for generalized cameras using the seventeen point method \cite{Pless:2003:Umcao}:
\begin{enumerate}
\item All of the observation rays pass through one common point before and after the camera motion.
\item The camera centres are on a line before and after the motion.
\item Each corresponding ray pair passes through the same local point in the general camera frame before and after the motion.
\end{enumerate}

For a camera cluster with non-overlapping FOV, it is possible that each component camera observes its own mutually exclusive set of feature points over the two keyframes. In this case, the system satisfies condition 3 and the solution to the seventeen point algorithm is always degenerate. However, it is known from previous results that in certain configurations, other solution methods are able to recover an accurate estimate of the motion and structure. Consequently, the seventeen point algorithm does not always recover a solution when one exists. This problem was noticed by \mtEtAl{Li}.~\cite{Li:2008:latmeugcm}, who have since modified the algorithm for use with non-overlapping clusters, but the subsequent degeneracy analysis has not been carried out. More importantly, the degenerate configurations are specific to the linear method of estimation. In this work, the minimization of image-space reprojection error is considered and the configurations for which an optimization of this type will fail are identified in the subsequent analysis.

% MULTICAMERA ESTIMATION

\section{Multicamera Cluster Pose Estimation}
\label{section:multicamera}
% Reference to projective stuff in appendix

\subsection{Projective Geometry}
The projective space $\mtProjective{n}$ (refer to \ref{section:projective} for a brief introduction) provides a convenient way of representing the camera measurement system in terms of homogeneous transformations and points \cite{Hartley:2003:Mvgcv}. It will sometimes be necessary to move between the respective real and projective space representation of points, and the following promotion and demotion operators are defined. The projective promotion operator $\mtHVec{\rho} : \mtReal{n} \rightarrow \mtProjective{n}$ maps a point $\mtVec{x}$ in the real vector space to its representation in the projective space,
\begin{align}
\mtHVec{\rho} \left( \mtVec{x} \right) &= \mtTranspose{ \begin{bmatrix} \mtTranspose{ \mtVec{x} } & 1 \end{bmatrix} }.  
\end{align}
The projective demotion operator $\mtVec{\pi}_n : \mtProjective{n} \rightarrow \mtReal{n}$ maps a point $\mtHVec{x}$ in the projective space back to the corresponding point $\mtVec{x}$ the real vector space,
\begin{align}
\mtVec{x} &= \mtVec{\pi}_n \left( \mtHVec{x} \right) \\
&= \begin{cases}
\text{undefined} & \mbox{if } x_{n+1} = 0 \\
\mtTranspose{ \begin{bmatrix} \dfrac{\mtHVar{x}_1}{\mtHVar{x}_{n+1}} & \dfrac{\mtHVar{x}_2}{\mtHVar{x}_{n+1}} & \dots & \dfrac{\mtHVar{x}_n}{\mtHVar{x}_{n+1}} \end{bmatrix} } & \mbox{if } x_{n+1} \neq 0.
\end{cases}
\end{align}
Note that the result of this operator is undefined for points at infinity.

In this work, unless it is ambiguous from the context, the promotion and demotion operators will be implied by the vector notation. The homogeneous coordinates for a given vector $\mtVec{x} \in \mtReal{n}$ will simply be written as $\mtHVec{x} \in \mtProjective{n}$, but implicitly, $\mtHVec{x} \equiv \mtHVec{\rho} \left( \mtVec{x} \right)$, and likewise, $\mtVec{x} \equiv \mtVec{\pi}_n \left( \mtHVec{x} \right)$ assuming $\mtHVar{x}_{n+1} \neq 0$.

\subsection{Pin-hole Camera Model}
\label{section-pinholecamera}
An individual component camera within the cluster is modelled as a simple pin-hole imaging device, which maps 3D points onto a 2D plane called the image plane \cite{Lu:2004:AsmDra}. An example is shown in Figure \ref{camera_model}. A 3D point $\mtHPoint{\mtCam{i}} = \mtTranspose{ \begin{bmatrix} \mtVar{x}^{\mtCam{i}} & \mtVar{y}^{\mtCam{i}} & \mtVar{z}^{\mtCam{i}} & 1 \end{bmatrix} }$, represented in the projective space $\mtProjective{3}$, and expressed with respect to the $i^{\text{{th}}}$ camera coordinate frame, $\mtCam{i}$, is projected onto the image plane $\mtImage{i}$. The intersection of the point feature ray $\mtHPoint{\mtCam{i}}$, through the optical centre, $\mtVec{o}_{i}$, with the image plane occurs at the point, $\mtTranspose{ \begin{bmatrix}
u &
v \end{bmatrix} } \in \mtReal{2}$. It is assumed that each camera has been intrinsically calibrated using one of the many existing offline techniques \cite{Hartley:2003:Mvgcv}, such that the measurements are made to match the structure shown.

%TODO: Add image plane markings (I_i)
\begin{figure*}[ht]
\centering
\includegraphics[width=0.8\textwidth]{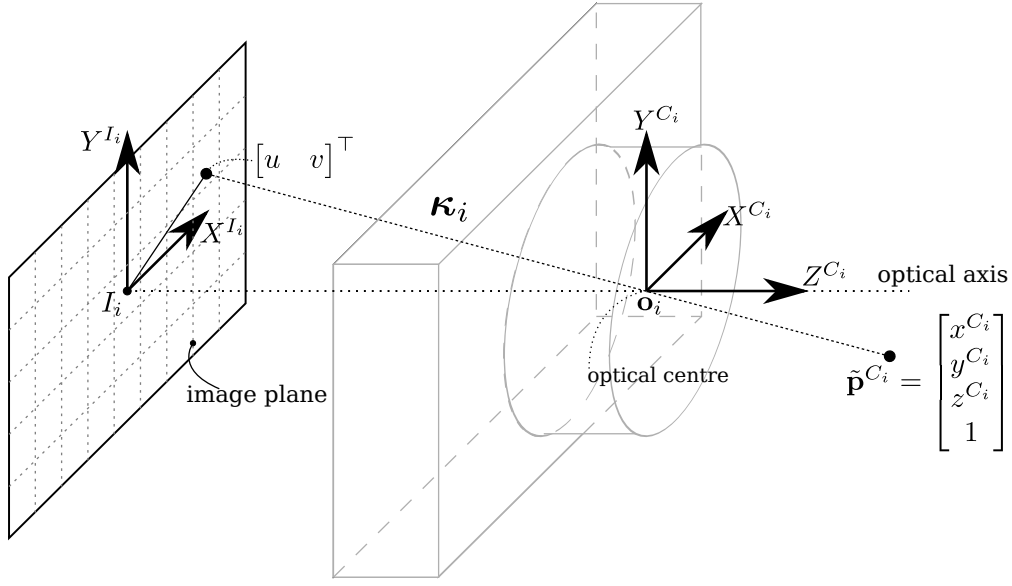}
\caption[Pin-hole Camera Model]{A simple pin-hole camera measurement model is used to relate the camera frame coordinates to the camera image plane coordinates for a feature point.}
\label{camera_model}
\end{figure*}

The camera projection matrix, $\mtProj{i}$, maps the point in $\mtProjective{3}$ into $\mtProjective{2}$ on the image plane. It is assumed in this work, without loss of generality, that the projection matrices for all of the cameras have the form,
\begin{align}
\mtProj{i} = \begin{bmatrix} -1 & 0 & 0 & 0 \\ 0 & -1 & 0 & 0 \\ 0 & 0 & 1 & 0 \end{bmatrix}.
\end{align}
The point $\mtHPoint{\mtCam{i}}$ is projected into $\mtProjective{2}$ on the image plane,
\begin{align}
\mtHPoint{\mtImage{i}} &= \mtProj{i} \mtHPoint{\mtCam{i}},
\end{align}
then subsequently mapped to the actual image plane coordinates in $\mtReal{2}$ through the demotion function $\mtVec{\pi}_2$ for $\mtProjective{2}$,
\begin{equation}
\label{eq:2dprojection}
\mtVec{\pi}_2(\mtHPoint{\mtImage{i}}) = 
\begin{bmatrix}
u \\
v
\end{bmatrix}
=
\begin{bmatrix}
\dfrac{-\mtVar{x}^{\mtCam{i}}}{\mtVar{z}^{\mtCam{i}}} \\[10pt]
\dfrac{-\mtVar{y}^{\mtCam{i}}}{\mtVar{z}^{\mtCam{i}}} \end{bmatrix}, \quad \mtVar{z}^{\mtCam{i}} \neq 0.
\end{equation}

Each camera is assumed to have an FOV strictly less than 180 degrees and therefore, is only able to observe points in front of the lens so every point is constrained to have a positive z-axis coordinate,
\begin{align}
{\mtVar{z}^{\mtCam{i}}} > 0,
\end{align}
which satisfies \eqref{eq:2dprojection}.

\subsection{Calibrated Multicamera Cluster}
Collectively, the calibrated camera cluster is modelled as a set of $n_c$ component pin-hole cameras with known relative coordinate transformations between each camera coordinate frame. Accordingly, a point $\mtHPoint{\mtCam{h}}$ in the camera frame $\mtCam{h}$, can be transformed into any other camera frame $\mtCam{i}$ by,
\begin{equation}
\mtHPoint{\mtCam{i}} = \mtT{\mtCam{h}}{\mtCam{i}} \mtHPoint{\mtCam{h}}
\end{equation}
where $\mtT{\mtCam{h}}{\mtCam{i}} \in \mtSE{3}$, $\forall i,h \in \{ 1,\dots,n_c \}$, is a homogeneous transformation matrix in $\mtSE{3}$ \cite{Murray:1994:amitrm}. Without loss of generality, the coordinate frame for the camera cluster is chosen to coincide with the first camera frame, $\mtCam{1}$. The transformation from camera $h$ to the cluster frame can be written in shortened form as $\mtTc{h} \equiv \mtT{\mtCam{h}}{\mtCam{1}}$, where the cluster frame $\mtFrame{C}_1$ is implied when the superscript is neglected. The transformation is shown in Figure \ref{fig:camera-calib}.

\begin{figure}[ht]
\centering
\includegraphics[width=0.35\textwidth]{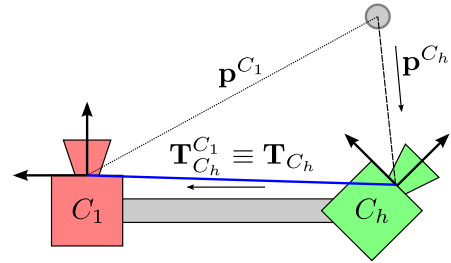}
\caption[Camera Cluster Calibration]{The relative position and orientation of each camera is known relative to the cluster frame, $\mtCam{1}$ and therefore, the position of points in any camera frame can be found with respect to the cluster frame using the known transformation, $\mtTc{h}$.}
\label{fig:camera-calib}
\end{figure}

\subsection{Point Feature Target Object Model}
% What are point features
The tracked target object or environment, henceforth referred to simply as the target, is a rigid body which contains a set of visible point features. A point feature is a visually distinguishable point on the tracked physical target that corresponds to a unique 3D position in a local target coordinate frame $\mtTarget$, and is measurable in a set of camera images through a relative motion sequence. Image measurements of these point features are extracted from the images using image processing techniques, including feature extraction algorithms like the FAST corner detector \cite{Rosten:2005:Fplhpt,Rosten:2006:Mlhscd}, the Scale-Invariant Feature Transform (SIFT) \cite{Lowe:1999:ORLSF}, or Speeded-Up Robust Features (SURF) \cite{Bay:2006:SSurf}. %An example of a moving target, a ship, with a set of corner point features is shown in Figure \ref{target-model}.

%\begin{comment}
%\begin{figure}[ht]
%\centering
%\includegraphics[width=0.4\textwidth]{ship-target.eps}
%\caption[Feature Point-Based Target Model]{The target object is a rigid body consisting of a set of point features, $\mtVec{p}^{\mtTarget}_j$, at fixed locations in a local target model coordinate frame, $M$. Point features can be identified and measured in the image plane of any observing cameras. In this example, the cameras on the aerial vehicle measure the image coordinates of visible corners on the moving ship.}
%\label{target-model}
%\end{figure}
%\end{comment}
%
%\begin{comment}
%The locations of the point features are constrained to be fixed with respect to each other. This allows for the relative target pose to be fully characterized by a single homogeneous transformation matrix in $\mtSE{3}$ representing the position and orientation of the local target model frame, $\mtTarget$, with respect to a reference coordinate frame, $\mtFrame{W}$,
%\begin{align}
%\mtT{\mtTarget}{\mtFrame{W}} = 
%\begin{bmatrix}
%\mtRx{\mtTarget}{\mtFrame{W}} & \mttx{\mtTarget}{\mtFrame{W}} \\
%\mtZeros{1}{3} & 1
%\end{bmatrix} \in \mtSE{3},
%\end{align}
%where $\mtRx{\mtTarget}{\mtFrame{W}} \in \mtSO{3}$, $\mttx{\mtTarget}{\mtFrame{W}} \in \mtReal{3}$, and $\mtMat{0}$ is the zero matrix with the specified dimensions.
%\end{comment}
%

The target model point features are organized into $n_k$ keyframes, each a six degree of freedom pose with respect to the target model reference frame $\mtTarget$, along with the $n_c$ images from the cluster cameras captured at that location, as in \cite{Klein:2007:PTMSAW} for a single camera. The coordinate frame of camera $h$ at keyframe $k$ is denoted $\mtCiKj{h}{k}$. 

Since the relative position and orientation of each camera within the cluster is fixed at all times, the $\mtIth{k}$ keyframe pose is parameterized by the single homogeneous transformation for the cluster coordinate frame at the keyframe, $\mtCiKj{1}{k}$, with respect to the target model reference frame, $\mtTarget$, resulting in $\mtT{\mtCiKj{1}{k}}{\mtTarget} \in \mtSE{3}$. The $\mtCam{1}$ and $\mtTarget$ frames are applied universally in this keyframe pose definition, and therefore, the transformation will be written simply as $\mtTk{k} \equiv \mtT{\mtCiKj{1}{k}}{\mtTarget}$. The pose of camera $h$ at keyframe $k$ is easily found as,
\begin{align}
\mtT{\mtCiKj{h}{k}}{\mtTarget} = \mtTk{k} \mtTc{h}.
\end{align}

The position of the $j^{\text{th}}$ point feature is parameterized by the azimuth and altitude angles of the vector from the origin of the anchor camera coordinate frame through the feature, $\mtVec{\mu}_j = \mtTranspose{[\phi_j,\theta_j]}$ where $\phi_j,\theta_j \in \left(-\frac{\pi}{2},\frac{\pi}{2}\right)$. The depth along this bearing to the point feature, is the value $s_j \in \mtReal{+}$. The bearing angles are used to form the unit vector in the camera coordinate frame at the first keyframe,
\begin{align}
\label{eq:observability:pointanchor}
\mtUnitVec{p}_j^{\mtCK{h}{1}} &= \begin{bmatrix}
\sin \phi_j \cos \theta_j \\
- \sin \theta_j \\
\cos \phi_j \cos \theta_j \end{bmatrix},
\end{align}
and the point feature position is along this bearing at the distance $s_j$,
\begin{align}
\label{eq:8}
\mtVec{p}_j^{\mtCK{h}{1}} &= s_j \mtUnitVec{p}_j^{\mtCK{h}{1}}.
\end{align}

% TODO: Track assumptions in a more formal way

An example system with a camera cluster composed of $n_c = 2$ cameras is shown in Figure \ref{fig:system-frames}. The cameras in this example are arranged back-to-back with the optical axes looking outwards along the green axes of the associated coordinate frames. The $\mtIth{j}$ point feature is anchored in the second camera at the first keyframe, $\mtCiKj{2}{1}$, and its position with respect to this coordinate frame is represented as $\mtPoint{\mtCK{2}{1}}_j$.

\begin{figure}[ht]
\centering
\includegraphics[width=0.48\textwidth]{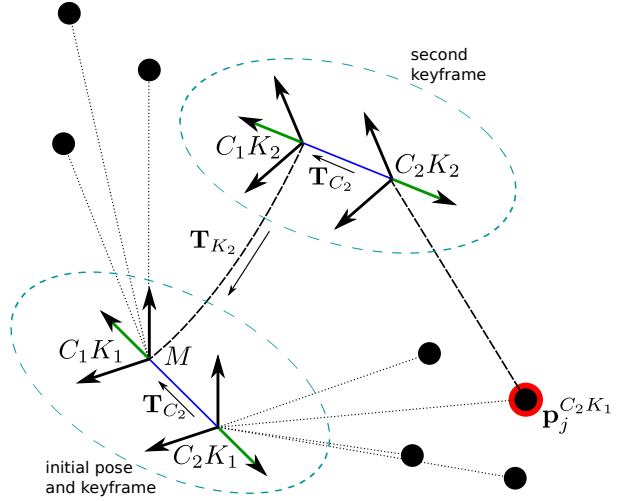}
\caption[Target Model Parameterization]{An example target object model with two keyframes for a two-camera back-to-back cluster. The cameras look outwards with the green arrows showing the optical axes. The point feature $j$ is anchored, and therefore, positioned within the $\mtCiKj{2}{1}$ coordinate frame. The relative pose of camera 2, $\mtTc{2}$, is known from calibration, but the relative pose of keyframe 2, $\mtTk{2}$, as well as the position of the point features must be estimated.}
\label{fig:system-frames}
\end{figure}

The parameters representing the poses of the keyframes, together with the positions of the point features observed within them, compose the target model, as well as the full system state. These parameters are estimated using the point feature image measurements within the cluster cameras.

\subsection{Multicamera Cluster SLAM System}
\label{subsection:slamsystem}

This work considers the motion and structure estimation for a cluster of $n_c$ cameras observing a set of $n_f$ point features over two keyframes. The cameras within the cluster are arranged with little or no overlap in their FOV where each point feature in the target model is visible in only one camera at the first keyframe. Without loss of generality, the target model frame is chosen to coincide with the pose of the first keyframe $\mtTarget \equiv \mtCiKj{1}{1}$. This results in the keyframe transformation becoming the identity,
\begin{align}
\mtTk{1} = \mtIdentity{4}.
\end{align}

Several further assumptions about the system are made to facilitate the analysis in the subsequent sections:
\begin{mtAssumption}
\label{assum:1}
Each point feature is observed and measured by only one of the component cameras at the first keyframe. The position of the point feature is expressed with respect to the coordinate frame for that camera at the first keyframe. The coordinate frame in which the point feature is parameterized is referred to as the anchor keyframe and camera coordinate frame.
\end{mtAssumption}
\begin{mtAssumption}
\label{assum:2}
Each point feature is observed and measured by one or more of the component cameras at the second keyframe. At least one of the observations is by a camera for which its motion is not collinear with the initial bearing to the point feature at the first keyframe. The camera and keyframe in which the observation occurs is called the observing keyframe and camera coordinate frame.
\end{mtAssumption}
\begin{mtAssumption}
\label{assum:3}
The point feature positions and keyframe poses are arranged such that if a camera observes a point feature, the feature position expressed in the observing camera coordinate frame has a finite positive non-zero z-axis component, $0<z<\infty$.
\end{mtAssumption}

For Assumption \ref{assum:1}, the function $h : \{1,\dots,n_f\} \rightarrow \{1,\dots,n_c\}$ maps the point feature index to the anchor camera index. As a result, the anchor camera for the $\mtIth{j}$ point feature is camera $h(j)$. In the following, when it is obvious from the context, the anchor camera index will be written in the shortened-form by dropping the argument, $h \equiv h(j)$.

Similarly for Assumption \ref{assum:2}, the $\mtIth{j}$ point feature is observed and measured by $n_o(j) \in \mtNatural$ cameras at the second keyframe. The indices of the observing cameras are found using the function $i: \{1,\dots,n_f\} \times \{1,\dots,n_o(j)\} \rightarrow \{1,\dots,n_c\}$, such that the $\mtIth{k}$ observation of the $\mtIth{j}$ point feature at the second keyframe is measured by camera $i(j,k)$. Once again, this will be shortened to exclude the feature index $j$ and observation index $k$ when it is implied by the context, $i \equiv i(k) \equiv i(j,k)$.

The motion of the camera cluster between the two keyframes is parameterized by six values describing the relative translation and orientation of the first keyframe with respect to the second keyframe. The translation parameters, $t_x$, $t_y$, $t_z$, form the relative translation vector, $\mttK = \mtTranspose{[t_x,t_y,t_z]}$, and the rotation parameters, $\mtVec{\omega}_{\mtFrame{K}} = \mtTranspose{[\omega_x,\omega_y,\omega_z]}$, form the relative rotation matrix, $\mtRK \in \mtSO{3}$. Together, the rotation and translation form the transformation, $\mtTK \in \mtSE{3}$.

The resulting state vector, $\mtVec{x} \in \mtReal{n}$, where $n = 6+3n_f$, is composed of the parameters for the $n_f$ point features, along with the relative translation and orientation states for the cluster motion between the keyframes,
\begin{align}
\mtVec{x} = \begin{bmatrix} 
\mtVec{x}_1 \\
\mtVec{x}_2 \\
\mtVec{x}_3 \\
\end{bmatrix},
\end{align}
where
\begin{align}
\mtVec{x}_1 = \mtTranspose{[s_1,\dots,s_{n_f}]} \in \mtReal{n_f}_{+},
\end{align}
are the radial distances to the point features,
\begin{align}
\mtVec{x}_2 = \mtTranspose{[ \mtTranspose{\mttK}, \mtTranspose{\mtVec{\omega}_K} ]} \in \mtReal{6},
\end{align}
are the relative position and orientation of the first keyframe with respect to the second keyframe and,
\begin{align}
\mtVec{x}_3 = \mtTranspose{[ \mtTranspose{\mtVec{\mu}_1}, \mtTranspose{\mtVec{\mu}_2},\dots,\mtTranspose{\mtVec{\mu}_{n_f}} ]} \in \mtReal{2n_f},
\end{align}
are the bearings to the point features in their respective anchor camera coordinate frames. This state order has been specifically chosen in order to facilitate the analysis of the degeneracies of the solution presented in Section \ref{section:degeneracies}.

%\begin{comment}
%The position of a point in the coordinate frame of camera $h$ at keyframe 1, ($\mtCiKj{h}{1}$), is denoted $\mtHVec{p}^{\mtCK{h}{1}} \in \mtProjective{3}$. The transformation from this coordinate frame into the coordinate frame of camera $i$ at keyframe $\ell=1,2$, ($\mtCiKj{i}{\ell}$), will be written,
%\begin{align}
%\label{eq:transrottrans}
%\mtT{\mtCK{h}{1}}{\mtCK{i}{\ell}} = \begin{bmatrix} 
%\mtRx{\mtCK{h}{1}}{\mtCK{i}{\ell}} & \mttx{\mtCK{h}{1}}{\mtCK{i}{\ell}} \\
%\mtZeros{1}{3} & 1 \end{bmatrix}. 
%\end{align}
%Therefore, the position of the point parameterized in $\mtCiKj{h}{1}$, from the perspective of keyframe $\mtCiKj{i}{\ell}$ is written,
%\begin{align}
%\mtHVec{p}^{\mtCK{i}{\ell}} = \mtT{\mtCK{h}{1}}{\mtCK{i}{\ell}} \mtHVec{p}^{\mtCK{h}{1}}.
%\end{align}
%\end{comment}

\subsection{Camera Cluster Measurement Model}
The system measurement vector, $\mtVec{z} \in \mtReal{m}$, is formed by stacking the image plane coordinates of all of the point feature observations in all cameras at both keyframes, where $m = 2( n_f+m_o)$ with
\begin{align}
m_o = \sum\limits^{n_f}_{j=1}  n_o(j),
\end{align}
as the total number of observations of all of the point features at the second keyframe. 

The measurement model, relating the observed point feature locations in the camera image planes, to the system states, can be written as a series of coordinate transformations. Suppose that the $\mtIth{j}$ point feature, anchored in the coordinate frame $\mtCiKj{h}{1}$, is measured by camera $i$ at $\mtCiKj{i}{2}$. An example of this chain of transformations is shown for the simple back-to-back two-camera cluster system in Figure \ref{fig:system-frames}. In this particular case, the point feature $j$ is anchored in $\mtCiKj{2}{1}$ and observed in $\mtCiKj{2}{2}$.

The point feature position parameters give the location of the $j^{\text{th}}$ feature in its anchor keyframe and camera frame $\mtCiKj{h}{1}$, resulting in $\mtVec{p}^{\mtCK{h}{1}}_j$. This point feature is first transformed into the target model coordinate frame by,
\begin{align}
\mtHPoint{\mtTarget}_j &= \mtTk{1} \mtTc{h} \mtHVec{p}^{\mtCK{h}{1}}_j\\
&= \mtTc{h} \mtHVec{p}^{\mtCK{h}{1}}_j,
\end{align}
which is the transformation provided by the known cluster calibration.

The point feature position, with position estimate expressed in the target model reference frame, is transformed into the coordinate frame of the observing keyframe and camera $\mtCiKj{i}{\ell}$ using the relative keyframe pose transformation, $\mtTk{\ell}$, and the cluster calibration,
\begin{align}
\mtHPoint{\mtCK{i}{\ell}}_j &= \mtTranspose{ \begin{bmatrix} x^{\mtCK{i}{\ell}}_j & y^{\mtCK{i}{\ell}}_j & z^{\mtCK{i}{\ell}}_j & 1 \end{bmatrix} } \\
&= \mtInverse{(\mtTc{i})} \mtInverse{(\mtTk{\ell})} \mtHPoint{\mtTarget}_j \\
\label{eq:3-16}
&= \mtInverse{(\mtTc{i})} \mtInverse{(\mtTk{\ell})} \mtTc{h} \mtHVec{p}^{\mtCK{h}{1}}_j.
\end{align}

Finally, the point is projected into $\mtProjective{2}$ and onto the image plane of camera $\mtCam{i}$ using the corresponding projection matrix, $\mtProj{i}$,
\begin{align}
\mtHVec{u}^{\mtCK{i}{\ell}}_j &= \mtTranspose{\begin{bmatrix} u_x & u_y & u_z \end{bmatrix}} \\
&= \mtProj{i} \mtHPoint{\mtCK{i}{\ell}}_j \\
&= \begin{bmatrix}
-x^{\mtCK{i}{\ell}}_j \\[4pt]
-y^{\mtCK{i}{\ell}}_j \\[4pt]
z^{\mtCK{i}{\ell}}_j \end{bmatrix},
\end{align}
which leads to the resulting measurement vector $\mtVec{z}^{\mtCK{i}{\ell}}_j \in \mtReal{2}$ and mapping $\mtVec{g}^{\mtCK{i}{\ell}}_j : \mtReal{n} \rightarrow \mtReal{2}$ for the observation of point feature $j$ in camera $i$ at keyframe $\ell$,
\begin{align}
\label{eq:25}
\mtVec{z}^{\mtCK{i}{\ell}}_j = \mtVec{g}^{\mtCK{i}{\ell}}_j(\mtVec{x}) 
&= \mtVec{\pi}_2 \left( \mtHVec{u}^{\mtCK{i}{\ell}}_j \right).
\end{align}

Each of the intermediate transformations in \eqref{eq:3-16} can be represented by a rotation matrix and translation vector,
\begin{align}
\label{eq:22}
\mtTc{h} &= \begin{bmatrix} \mtRc{h} & \mttc{h} \\
\mtZeros{1}{3} & 1 \end{bmatrix} \\
\mtTk{\ell} &= \begin{bmatrix} \mtRk{\ell} & \mttk{\ell} \\
\mtZeros{1}{3} & 1 \end{bmatrix} =  
\begin{cases}
\mtIdentity{4}, & \ell = 1 \\
\begin{bmatrix} \mtRKt & -\mtRKt \mttK \\
\mtZeros{1}{3} & 1 \end{bmatrix}, & \ell = 2 \end{cases} \\
\label{eq:24}
\mtTc{i} &= \begin{bmatrix} \mtRc{i} & \mttc{i} \\
\mtZeros{1}{3} & 1 \end{bmatrix}.
\end{align}
When \eqref{eq:22}--\eqref{eq:24} are substituted into \eqref{eq:3-16} along with \eqref{eq:8}, the coordinates of the point feature position in $\mtReal{3}$ become,
\begin{align}
\label{eq:observability:simpletrans2}
\mtPoint{\mtCK{i}{\ell}}_j &= 
s_j \mtRct{i} \mtRkt{\ell} \mtRc{h} \mtUnitVec{p}^{\mtCK{h}{1}}_j 
- \mtRct{i} \mtRkt{\ell} \mttk{\ell} 
+ \mtRct{i} \mtRkt{\ell} \mttc{h} 
- \mtRct{i} \mttc{i}
\\
&= \mtRct{i} \left( 
s_j \mtRkt{\ell} \mtRc{h} \mtUnitVec{p}^{\mtCK{h}{1}}_j 
- \mtRkt{\ell} \mttk{\ell} 
+ \mtRkt{\ell} \mttc{h} 
- \mttc{i} 
\right) \\
& = \mtRct{i} \mtVec{q}^{\mtCK{i}{\ell}}_j,
\end{align}
where
\begin{align}
\mtVec{q}^{\mtCK{i}{\ell}}_j &= s_j \mtUnitVec{a}_{j,\ell} + \mtVec{b}_{\ell} + \mtVec{c}_{h,\ell} + \mtVec{d}_{i}
\end{align}
with
\begin{align}
\mtUnitVec{a}_{j,\ell} &= \mtRkt{\ell} \mtRc{h} \mtUnitVec{p}^{\mtCK{h}{1}}_j \\
\mtVec{b}_{\ell} &= -\mtRkt{\ell} \mttk{\ell} \\
\mtVec{c}_{h,\ell} &=  \mtRkt{\ell} \mttc{h} \\
\mtVec{d}_{i} &= -\mttc{i}
\end{align}
and
\begin{align}
\mtRc{i} = \begin{bmatrix} \mtUnitVec{n}_{i,x} & \mtUnitVec{n}_{i,y} & \mtUnitVec{n}_{i,z} 
\end{bmatrix} \in \mtSO{3},
\end{align}
where $\mtUnitVec{n}_{i,x}$, $\mtUnitVec{n}_{i,y}$, and $\mtUnitVec{n}_{i,z}$ are the orthonormal basis vectors for the observing camera $i$ frame with respect to the camera 1 coordinate frame. An example system consisting of the cameras observing point features over two keyframes is shown in Figure~\ref{fig:obs-setup} with the intermediate variables labelled.

\begin{figure*}[!ht]
\centering
\includegraphics[width=0.8\textwidth]{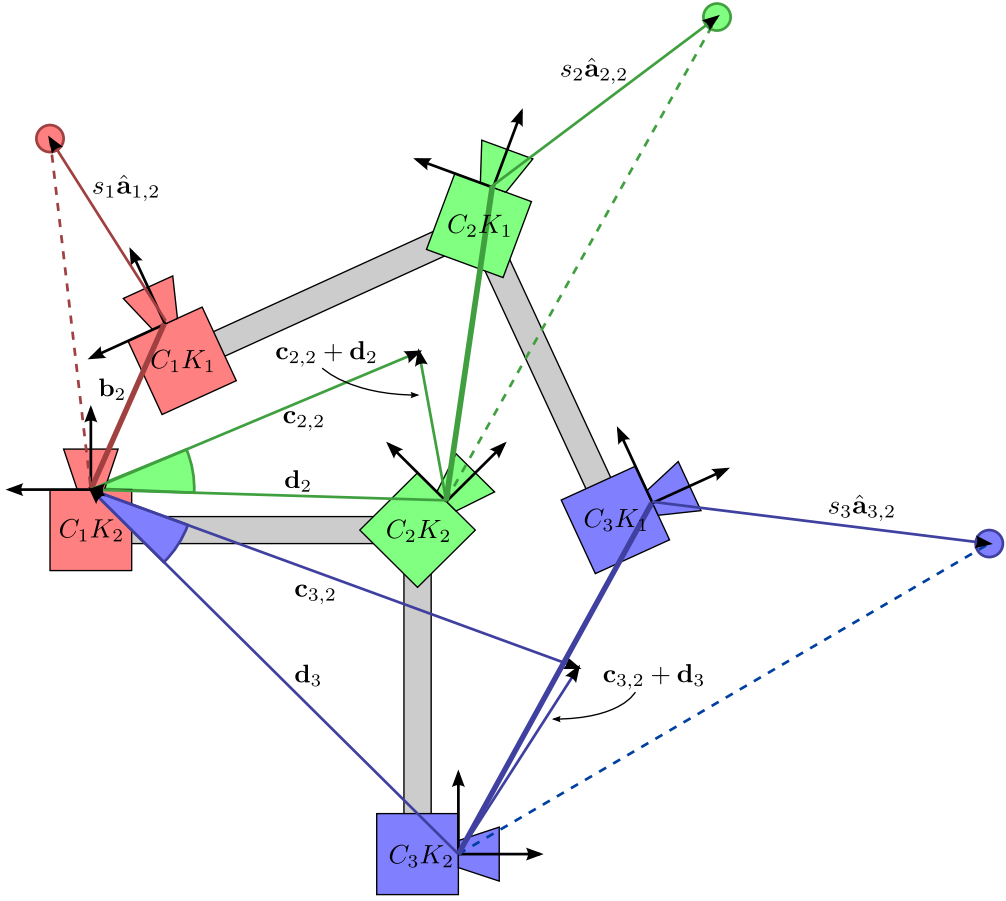}
\caption{An example three-camera cluster observing point features over two keyframes with intermediate vectors labelled. Vectors are parameterized with reference to keyframe $\mtCiKj{1}{2}$.}
\label{fig:obs-setup}
\end{figure*}

The set of camera observation vectors for point feature $j$ is defined as the displacements between the anchor camera coordinate frame at the first keyframe, $\mtCiKj{h}{1}$, and the centres of each of the observing cameras at the second keyframe, $\mtCiKj{i(k)}{2}$, $\forall k \in \{1,\dots,n_o(j)\}$. The set of vectors are,
\begin{align}
V = \{ \mtVec{v}_{\alpha,\beta} \in \mtReal{3} | \alpha,\beta \in \mtNatural, \alpha \leq n_f \text{ and } \beta \leq n_o(\alpha) \}.
\end{align}
Therefore, if it is included in the set $V$, the camera observation vector can be written as,
\begin{align}
\mtVec{v}_{\alpha,\beta}  = \mtVec{b}_{2} + \mtVec{c}_{h(\alpha),2} + \mtVec{d}_{i(\beta)}.
\end{align}
and is illustrated in the example system shown in Figure~\ref{fig:cross-cam}.

\begin{figure}[!ht]
\centering
\includegraphics[width=0.4\textwidth]{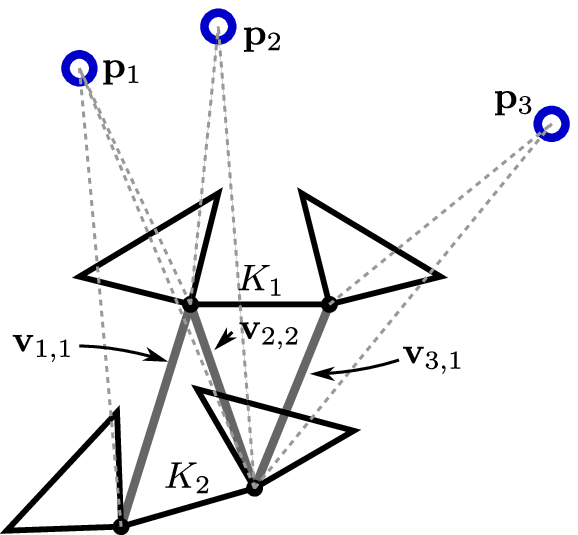}
\caption{An example two-camera cluster observing three features. Each feature is observed in one camera at the first keyframe, but some are seen by different cameras at the second keyframe. The camera observation vectors, $\mtVec{v}_{1,1}$, $\mtVec{v}_{2,2}$, $\mtVec{v}_{3,1}$ link the cameras which see the same feature at the different keyframes.}
\label{fig:cross-cam}
\end{figure}

The image coordinates of each individual point feature observation measurements are then compiled into a vector for point feature $j$ containing all of the individual observations of that feature at both keyframes,
\begin{align}
\mtVec{z}_j &= 
\begin{bmatrix}
\mtVec{z}^{\mtCK{i(1)}{2}}_j \\
\vdots \\
\mtVec{z}^{\mtCK{i(n_o)}{2}}_j \\
\mtVec{z}^{\mtCK{h}{1}}_j
\end{bmatrix} \in \mtReal{2+2n_o},
\end{align}
where $n_o \equiv n_o(j)$ is the number of observations of the $\mtIth{j}$ point feature at the second keyframe.

The full system measurement vector is composed of the observations of the $n_f$ point features at both keyframes,
\begin{align}
\label{eq:system-measurement}
\mtVec{z} &= \begin{bmatrix} \mtVec{z}_1 \\
\vdots \\
\mtVec{z}_{n_f} \end{bmatrix} \in \mtReal{m}.
\end{align}

% DEGENERATE CONFIGURATIONS
\section{Degeneracy Analysis}
\label{section:degeneracies}

\subsection{Solution Degeneracies}
Typical optimization methods attempt to minimize a nonlinear cost function, $c : \mtReal{n} \rightarrow \mtReal{}$, and determine the optimal state vector estimate, $\mtEstVec{x}^{*} \in \mtReal{n}$, such that,
\begin{align}
\mtEstVec{x}^{*} = \mtArgMin{\mtVec{x}} c(\mtVec{x}).
\end{align}
The optimization proceeds iteratively, starting with an initial state estimate, $\mtEstVec{x}_0 \in \mtReal{n}$. Each iteration seeks to update the current state estimate, $\mtEstVec{x}_k$, with a vector $\mtVec{\delta}_k \in \mtReal{n}$,
\begin{align}
\mtEstVec{x}_{k+1} = \mtEstVec{x}_k + \mtVec{\delta}_k,
\end{align}
such that the sequence $\{ \mtEstVec{x}_0, \mtEstVec{x}_{1}, \mtEstVec{x}_{2}, \dots \} \rightarrow \mtEstVec{x}^{*}$. In this analysis, a cost function relating to sum of squared measurement error is assumed,
\begin{align}
c(\mtEstVec{x}_k) = \dfrac{1}{2} \mtTranspose{\mtErrVec{z}_k} \mtErrVec{z}_k,
\end{align}
where $\mtErrVec{z}_k = \mtVec{z} - \mtVec{g}(\mtEstVec{x}_k) \in \mtReal{m}$ is the measurement error vector at iteration $k$. A commonly-used method for BA is the Levenberg-Marquardt (LM) method \cite{Hartley:2003:Mvgcv}, although other optimization methods may be used such as Gauss-Newton, gradient descent, or Newton step. Each of these optimization methods can operate using the sum of squared reprojection error and the parameter update $\mtVec{\delta}_k$ is defined as the solution to,
\begin{align}
\label{eq:39}
\mtMat{N} \mtVec{\delta}_k = 
\mtTranspose{\mtMat{J}} \mtErrVec{z}_k
\end{align}
where $\mtMat{N}$ is the normal matrix, which varies by optimization method, and $\mtMat{J}$ is the measurement Jacobian such that,
\begin{align}
\label{eq:47}
\mtMat{J} = \mtEvaluated{\mtPartial{\mtVec{g}(\mtVec{x})}{\mtVec{x}}}{\mtVec{x} = \mtEstVec{x}_k} \in \mtReal{\mtBy{m}{n}}.
\end{align}
Solving for $\mtVec{\delta}_k$, the solution becomes,
\begin{align}
\label{eq:3-37}
\mtVec{\delta}_k = \mtInverse{\mtMat{N}}
\mtTranspose{\mtMat{J}} \mtErrVec{z}_k,
\end{align}
where $\mtErrVec{z}_k = \mtVec{z} - \mtVec{g}(\mtEstVec{x}_k)$ is the measurement error. A unique $\mtVec{\delta}_k$ can be found as long as the matrix $\mtMat{N}$ is invertible and the Jacobian has full rank. Therefore, the system in \eqref{eq:39} is degenerate when,
\begin{align}
\mtRank{\mtMat{J}} < n,
\end{align}
and the solution is under-constrained. The cases where the system becomes degenerate are the focus of this work and are investigated in the next section.

\subsection{Identification of Degenerate Configurations}
This section identifies the configurations of the camera cluster geometry, relative motion, and target model structure, for which the Jacobian $\mtMat{J}$ falls below full column rank. It is shown that for the assumptions stated previously, the $\mtBy{m}{n}$ measurement Jacobian matrix $\mtMat{J}$ has full rank if and only if a $\mtBy{m_o}{6}$ matrix, $\mtMat{M}_2$ has full rank.

To determine the rank of the Jacobian matrix, the structure of the sub-blocks formed for the point feature observations is investigated. Each point feature $j$ is observed by only one camera at the first keyframe. This observation adds two rows to the Jacobian,
\begin{align}
\mtJobs{j}{h}{1} &= \mtEvaluated{ \mtPartial{\mtVec{g}_j^{\mtCK{h}{1}}(\mtVec{x})}{\mtVec{x}} }{\mtVec{x} = \mtEstVec{x}} \\
&= \begin{bmatrix} 
\mtZeros{2}{n_f+6} & \mtZeros{2}{2(j-1)} & 
\mtEvaluated{ \mtPartial{\mtVec{g}_j^{\mtCK{h}{1}}(\mtVec{x})}{\mtVec{\mu}_j} }{\mtVec{x} = \mtEstVec{x}}
& \mtZeros{2}{2(n_f-j)}
\end{bmatrix},
\end{align}
where the only non-zero elements are in the $\mtBy{2}{2}$ block relating the measurement coordinates to the point feature bearing states,
\begin{align}
\label{eq:52}
\mtEvaluated{ \mtPartial{\mtVec{g}_j^{\mtCK{h}{1}}(\mtVec{x})}{\mtVec{\mu}_j} }{\mtVec{x} = \mtEstVec{x}} = 
\begin{bmatrix}
-\tan(\phi_j)^2 - 1 & 0 \\[10pt]
\dfrac{\tan(\phi_j) \tan(\theta_j)}{\cos(\phi_j)} & \dfrac{1}{\cos(\phi_j) \cos(\theta_j)^2}
\end{bmatrix}.
\end{align}
The non-zero element structure of these rows are shown in Figure~\ref{fig:Jjh1}.
% Figure of jacobian block at h,1
\begin{figure}[!ht]
\centering
\includegraphics[width=0.48\textwidth]{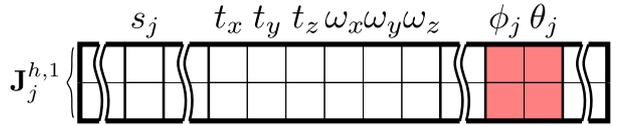}
\caption{Structure of the Jacobian rows for the measurement of the point feature $j$ in its anchor camera at the first keyframe. Non-zero elements are shown as the shaded cells. The columns not related to feature $j$ contain only zeros and have been removed for conciseness.}
\label{fig:Jjh1}
\end{figure}

Each point feature $j$ is also observed and measured by at least one camera at the second keyframe. The Jacobian matrix rows corresponding to the $\mtIth{k}$ observation of point feature $j$ in the second keyframe with camera $i(k)$ are the partial derivatives of the measurement equation \eqref{eq:25} with respect to the system states,
\begin{align}
\mtJobs{j}{i(k)}{2} &= \mtEvaluated{ \mtPartial{\mtVec{g}_j^{\mtCK{i(k)}{2}}(\mtVec{x})}{\mtVec{x}} }{\mtVec{x} = \mtEstVec{x}}, \\
\label{eq:jobs1}
&= \left( \mtEvaluated{ \mtPartial{\mtVec{\pi}_2 \left( \mtHVec{u}_j^{\mtCK{i(k)}{2}} \right)}{\mtHVec{u}_j^{\mtCK{i(k)}{2}}} }{\mtVec{x} = \mtEstVec{x}} \right) \left( \mtEvaluated{ \mtPartial{\mtHVec{u}_j^{\mtCK{i(k)}{2}}}{\mtVec{x}} }{\mtVec{x} = \mtEstVec{x}} \right),
\end{align} 
using the chain rule. Dropping the implied $\mtVec{x} = \mtEstVec{x}$, the first term evaluates to,
\begin{align}
\mtPartial{\mtVec{\pi}_2 \left( \mtHVec{u}_j^{\mtCK{i(k)}{2}} \right)}{\mtHVec{u}_j^{\mtCK{i(k)}{2}}} &= \dfrac{1}{\left( u_z \right)^2} \begin{bmatrix} u_z & 0 & -u_x \\ 0 & u_z & -u_y \end{bmatrix} \\
\label{eq:jobs2}
&= \dfrac{1}{\left( z_j^{\mtCK{i(k)}{2}} \right)^2} 
\begin{bmatrix} 
0 & -1 & 0 \\
1 & 0 & 0
\end{bmatrix}
\mtSkewSym{\mtHVec{u}_j^{\mtCK{i(k)}{2}}},
\end{align}
where $\mtSkewSym{\mtVec{a}}$ is the skew-symmetric matrix such that $\mtSkewSym{\mtVec{a}} \mtVec{b} = \mtCross{\mtVec{a}}{\mtVec{b}}$, $\forall \mtVec{a},\mtVec{b} \in \mtReal{3}$.

Substituting \eqref{eq:jobs2} back into \eqref{eq:jobs1} and recognizing that,
\begin{align}
\mtHVec{u}_j^{\mtCK{i(k)}{2}} = \begin{bmatrix} -1 & 0 & 0 \\ 0 & -1 & 0 \\ 0 & 0 & 1 \end{bmatrix} \mtRct{i(k)} \mtVec{q}_j^{\mtCK{i(k)}{2}},
\end{align}
the Jacobian rows can be written as,
\begin{align}
\mtJobs{j}{i(k)}{2}&= 
\dfrac{1}{\left( z_j^{\mtCK{i(k)}{2}} \right)^2} 
\begin{bmatrix} 
0 & 1 & 0 \\
-1 & 0 & 0
\end{bmatrix}
\mtSkewSym{\mtHVec{u}_j^{\mtCK{i(k)}{2}}}
\mtPartial{\mtHVec{u}_j^{\mtCK{i(k)}{2}}}{\mtVec{x}}, \\
&=\dfrac{1}{\left( z_j^{\mtCK{i(k)}{2}} \right)^2} 
\begin{bmatrix} 
0 & -1 & 0 \\
1 & 0 & 0
\end{bmatrix}
\mtRct{i(k)}
\mtSkewSym{\mtVec{q}_j^{\mtCK{i(k)}{2}}}
\mtPartial{\mtVec{q}_j^{\mtCK{i(k)}{2}}}{\mtVec{x}}, \\
&= \dfrac{1}{\left( z_j^{\mtCK{i(k)}{2}} \right)^2} 
\begin{bmatrix} 
-\mtTranspose{\left( \mtCross{\mtUnitVec{n}_{i(k),y}}{\mtVec{q}_j^{\mtCK{i(k)}{2}}} \right)} \\
\mtTranspose{\left( \mtCross{\mtUnitVec{n}_{i(k),x}}{\mtVec{q}_j^{\mtCK{i(k)}{2}}} \right)}
\end{bmatrix}
\mtPartial{\mtVec{q}_j^{\mtCK{i(k)}{2}}}{\mtVec{x}},
\end{align}
where the partial derivatives of the point feature position at the second keyframe with respect to the system states are written,
\begin{align}
\mtPartial{\mtVec{q}_j^{\mtCK{i(k)}{2}}}{\mtVec{x}} = 
\left[ \begin{array}{ccccc}
\mtZeros{3}{(j-1)} & \mtPartial{\mtVec{q}_j^{\mtCK{i(k)}{2}}}{s_j} & \mtZeros{3}{(n_f - j)} & \mtPartial{\mtVec{q}_j^{\mtCK{i(k)}{2}}}{\mttK} & \mtPartial{\mtVec{q}_j^{\mtCK{i(k)}{2}}}{\mtVec{\omega}_{\mtFrame{K}}}
\end{array} \right.& \nonumber \\
\left. \begin{array}{ccc}
\mtZeros{3}{2(j-1)} & \mtPartial{\mtVec{q}_j^{\mtCK{i(k)}{2}}}{\mtVec{\mu}_j} & \mtZeros{3}{2(n_f-j)}
\end{array} \right]&
\end{align}
with the position change with respect to the radial distance,
\begin{align}
\mtPartial{\mtVec{q}_j^{\mtCK{i(k)}{2}}}{s_j} = \mtUnitVec{a}_{j,2},
\end{align}
the translation between the keyframes,
\begin{align}
\mtPartial{\mtVec{q}_j^{\mtCK{i(k)}{2}}}{\mttK} = \mtIdentity{3},
\end{align}
the rotation between the keyframes,
\begin{align}
\mtPartial{\mtVec{q}_j^{\mtCK{i(k)}{2}}}{\mtVec{\omega}_{\mtFrame{K}}} 
= -\mtSkewSym{s_j \mtUnitVec{a}_{j,2} + \mtVec{c}_{h,2}},
\end{align}
and the initial bearings to the point feature,
\begin{align}
\mtPartial{\mtVec{q}_j^{\mtCK{i(k)}{2}}}{\mtVec{\mu}_j}
= s_j \mtRK \mtRc{h} 
\begin{bmatrix}
  \cos \phi_j \cos \theta_j   & -\sin \phi_j \sin \theta_j\\
  0                           & -\cos \theta_j \\
  -\sin \phi_j \cos \theta_j  & -\cos \phi_j \sin \theta_j
\end{bmatrix}.
\end{align}
The structure of the Jacobian rows associated with the observations of point feature $j$ at the second keyframe is shown in Figure~\ref{fig:Jji2}.

\begin{figure}[!ht]
\centering
\includegraphics[width=0.48\textwidth]{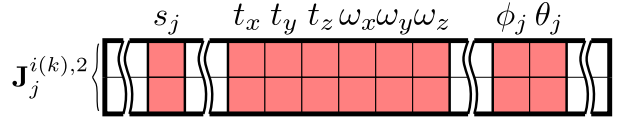}
\caption{Structure of the Jacobian rows for an observation of the point feature $j$ at the second keyframe.}
\label{fig:Jji2}
\end{figure}

The full measurement Jacobian is formed by stacking all of the observations of all of the point features at both keyframes,
\begin{align}
\label{eq:66}
\mtMat{J} = \begin{bmatrix}
\begin{bmatrix}
\mtJobs{1}{i(1,1)}{2} \\
\vdots \\
\mtJobs{1}{i(1,n_o(1))}{2} \\
\mtJobs{1}{h(1)}{1} \\
\end{bmatrix} \\
\vdots \\
\begin{bmatrix}
\mtJobs{n_f}{i(n_f,1)}{2} \\
\vdots \\
\mtJobs{n_f}{i(n_f,n_o(n_f))}{2} \\
\mtJobs{n_f}{h(n_f)}{1} \\
\end{bmatrix}
\end{bmatrix}.
\end{align}

%%%%%%%%%%%%%%%%%%%%%%%%%%%%%%%%%%%%%%%%%%%%%%%%%%%%%%%%%%%
% Start of Lemma(?)
The configurations for which this measurement Jacobian possesses full rank can be identified by checking the rank of a reduced-dimension matrix, as shown in the following Lemma.

\begin{mtLemma}
\label{lemma:1}
For a multicamera cluster SLAM system satisfying Assumptions \ref{assum:1}, \ref{assum:2}, and \ref{assum:3}, the rank of the measurement Jacobian matrix $\mtMat{J}$ in \eqref{eq:66} is full if and only if the rank of the matrix,
\begin{align}
\label{eq:m2}
\mtMat{M}_2 = \begin{bmatrix}
\begin{bmatrix}
-\mtTranspose{\mtUnitVec{a}_{1,2}} \mtSkewSym{\mtVec{v}_{1,1}} &
\mtTranspose{\mtUnitVec{a}_{1,2}} \mtSkewSym{\mtVec{v}_{1,1}} \mtSkewSym{\mtVec{w}_{1}} \\
\vdots & \vdots \\
-\mtTranspose{\mtUnitVec{a}_{1,2}} \mtSkewSym{\mtVec{v}_{1,n_o(1)}} &
\mtTranspose{\mtUnitVec{a}_{1,2}} \mtSkewSym{\mtVec{v}_{1,n_o(1)}} \mtSkewSym{\mtVec{w}_{1}} 
\end{bmatrix} \\
\vdots \\
\begin{bmatrix}
-\mtTranspose{\mtUnitVec{a}_{n_f,2}} \mtSkewSym{\mtVec{v}_{n_f,1}} &
\mtTranspose{\mtUnitVec{a}_{n_f,2}} \mtSkewSym{\mtVec{v}_{n_f,1}} \mtSkewSym{\mtVec{w}_{n_f}} \\
\vdots & \vdots \\
-\mtTranspose{\mtUnitVec{a}_{n_f,2}} \mtSkewSym{\mtVec{v}_{n_f,n_o(n_f)}} &
\mtTranspose{\mtUnitVec{a}_{n_f,2}} \mtSkewSym{\mtVec{v}_{n_f,n_o(n_f)}} \mtSkewSym{\mtVec{w}_{n_f}}
\end{bmatrix}
\end{bmatrix},
\end{align}
where
\begin{align*}
\mtVec{w}_{j} = s_j \mtUnitVec{a}_{j,2} + \mtVec{c}_{h(j),2},
\end{align*}
is full.
\end{mtLemma}

\begin{proof}
The strategy is to first show that the columns of $\mtMat{J}$ corresponding to the point feature positions ($s_j$, $\mtVec{\mu}_j$) have full rank. Consequently, the only way for the Jacobian to have less than full rank is when the columns corresponding to the keyframe motion ($\mttK$, $\mtRK$) have rank less than six.

By Assumption \ref{assum:3}, the position of the $\mtIth{j}$ point feature in its anchor camera and keyframe ensures that $\cos(\phi_j) > 0$ and $\cos(\theta_j) > 0$. As a result, the block \eqref{eq:52} always has a rank of 2 since the determinant is non-zero,
\begin{align}
\mtDet{\mtEvaluated{ \mtPartial{\mtVec{g}_j^{\mtCK{h}{1}}(\mtVec{x})}{\mtVec{\mu}_j} }{\mtVec{x} = \mtEstVec{x}}} = 
\dfrac{-1}{\cos(\phi_j)^3 \cos(\theta_j)^2} \neq 0.
\end{align}
Therefore, it is possible to diagonalize the sub-block using elementary row and column operations without changing the rank of the matrix. After diagonalization, the new matrix rows, $\mtMat{K}^{\mtCK{h}{1}}_j$, have the structure shown in Figure~\ref{fig:Jjh1_diag}. As a result, the columns corresponding to the bearing states $\mtVec{\mu}_j$ have full rank for all of the point features.

\begin{figure}[!ht]
\centering
\includegraphics[width=0.48\textwidth]{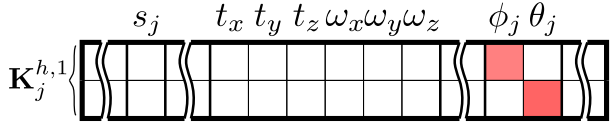}
\caption{Structure of the Jacobian rows for the measurement of the point feature $j$ in the first keyframe after diagonalizing the bearing sub-block.}
\label{fig:Jjh1_diag}
\end{figure}

Additionally, the columns associated with the point feature radial depth parameter $s_j$ for an observation at the second keyframe contain only zeros when,
\begin{align}
\begin{bmatrix}
-\mtTranspose{\mtUnitVec{n}_{i(k),y}} \\
 \mtTranspose{\mtUnitVec{n}_{i(k),x}}
\end{bmatrix}
\mtSkewSym{\mtVec{v}_{j,k}} \mtUnitVec{a}_{j,2} = \mtZeros{2}{1},
\end{align}
the displacement between the anchor and observing camera coordinate frames is collinear with the initial bearing to the point feature in the anchor camera frame. In this case, there is no information about the depth of the feature within this measurement since the triangulation baseline has zero length. However, by Assumptions \ref{assum:2} and \ref{assum:3}, there exists an observation $k \in \{ 1,\dots,n_o(j)\}$ such that,
\begin{align}
\begin{bmatrix}
s_x \\ s_y
\end{bmatrix} = 
\begin{bmatrix}
-\mtTranspose{\mtUnitVec{n}_{i(k),y}} \\
 \mtTranspose{\mtUnitVec{n}_{i(k),x}}
\end{bmatrix}
\mtSkewSym{\mtVec{v}_{j,k}} \mtUnitVec{a}_{j,2} \neq \mtZeros{2}{1},
\end{align}
and therefore, at least one non-zero element in the column. The matrix rows $\mtJobs{j}{i(k)}{2}$ are manipulated using the row operations matrix,
\begin{align}
\mtMat{O}^{i(k),2}_j = \begin{cases}
\begin{bmatrix}
 \mtTranspose{\mtUnitVec{n}_{i(k),x}} \mtSkewSym{\mtVec{q}_j^{\mtCK{i(k)}{2}}} \mtUnitVec{a}_{j,2} & 0 \\
 -\mtTranspose{\mtUnitVec{n}_{i(k),x}} \mtSkewSym{\mtVec{q}_j^{\mtCK{i(k)}{2}}} \mtUnitVec{a}_{j,2}  &  -\mtTranspose{\mtUnitVec{n}_{i(k),y}} \mtSkewSym{\mtVec{q}_j^{\mtCK{i(k)}{2}}} \mtUnitVec{a}_{j,2} \end{bmatrix}, & \text{if } s_x,s_y \neq 0 \\
\begin{bmatrix}
0 & 1 \\
 \mtTranspose{\mtUnitVec{n}_{i(k),x}} \mtSkewSym{\mtVec{q}_j^{\mtCK{i(k)}{2}}} \mtUnitVec{a}_{j,2}  & 0 \end{bmatrix}, & \text{if } s_x = 0, s_y \neq 0 \\
\begin{bmatrix}
1 & 0 \\
0 &  -\mtTranspose{\mtUnitVec{n}_{i(k),y}} \mtSkewSym{\mtVec{q}_j^{\mtCK{i(k)}{2}}} \mtUnitVec{a}_{j,2} 
\end{bmatrix}, & \text{if } s_x \neq 0, s_y=0 \\
\mtIdentity{2}, & \text{if } s_x,s_y = 0,
\end{cases}
\end{align}
in order to achieve the desired structure $\mtMat{K}^{i(k),2}_j$ for the Jacobian rows,
\begin{align}
\mtMat{K}^{i(k),2}_j = \mtVar{z}^{\mtCK{i(k)}{2}}_j \mtMat{O}^{i(k),2}_j \mtJobs{j}{i(k)}{2}.
\end{align}
which is shown in Figure~\ref{fig:Kji2}. 

\begin{figure}[!ht] %% Add in the second row labels k_{s_t}, etc.
\centering
\includegraphics[width=0.48\textwidth]{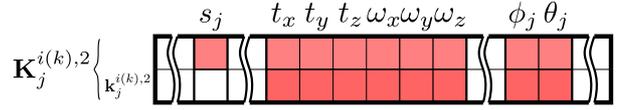}
\caption{Element structure of the modified Jacobian rows associated with measuring point feature $j$ at the second keyframe.}
\label{fig:Kji2}
\end{figure}

The second row of the matrix $\mtMat{K}^{i(k),2}_j$, labelled $\mtVec{k}^{i(k),2}_j$, becomes,
\begin{align}
\mtVec{k}^{i(k),2}_j &= \begin{bmatrix}
\mtVec{0} & k_{s_j} & \mtVec{0} & \mtVec{k}_{\mttK} & & \mtVec{k}_{\mtVec{\omega}_K} & \mtVec{0} & \mtVec{k}_{\mtVec{\mu}_j} & \mtVec{0}
\end{bmatrix} \\
 &= \dfrac{1}{\mtVar{z}^{\mtCK{i(k)}{2}}_j} \left( \mtTranspose{\mtUnitVec{n}_{i(k),z}} \mtVec{q}_j^{\mtCK{i(k)}{2}} \right) \mtTranspose{ \left( \mtCross{\mtVec{q}_j^{\mtCK{i(k)}{2}}}{\mtUnitVec{a}_{j,2}} \right) }
\mtPartial{\mtVec{q}_j^{\mtCK{i(k)}{2}}}{\mtVec{x}} \\
&= \mtTranspose{ \left( \mtSkewSym{ \mtVec{b}_2 + \mtVec{c}_{h,2} + \mtVec{d}_{i} } \mtUnitVec{a}_{j,2} \right) } \mtPartial{\mtVec{q}_j^{\mtCK{i(k)}{2}}}{\mtVec{x}}.
\end{align}
where the element associated with the radial distance parameter is zero,
\begin{align}
k_{s_j} &= \mtTranspose{ \left( \mtSkewSym{ \mtVec{b}_2 + \mtVec{c}_{h,2} + \mtVec{d}_{i} } \mtUnitVec{a}_{j,2} \right) } \mtUnitVec{a}_{j,2} \\
&= \mtTranspose{ \left( \mtVec{b}_2 + \mtVec{c}_{h,2} + \mtVec{d}_{i} \right) } \mtSkewSym{ \mtUnitVec{a}_{j,2} } \mtUnitVec{a}_{j,2} \\
&= 0,
\end{align}
and the columns for the keyframe motion states are now,
\begin{align}
\mtVec{k}_{\mttK} &=  - \mtTranspose{\mtUnitVec{a}_{j,2}}  \mtSkewSym{ \mtVec{b}_2 + \mtVec{c}_{h,2} + \mtVec{d}_{i(k)} } \\
&= - \mtTranspose{\mtUnitVec{a}_{j,2}} \mtSkewSym{\mtVec{v}_{j,k}},
\end{align}
and,
\begin{align}
\mtVec{k}_{\mtVec{\omega}_K} &=   \mtTranspose{\mtUnitVec{a}_{j,2}} \mtSkewSym{ \mtVec{b}_2 + \mtVec{c}_{h,2} + \mtVec{d}_{i(k)} } \mtSkewSym{ s_j \mtUnitVec{a}_{j,2} + \mtVec{c}_{h,2} } \\
&= \mtTranspose{\mtUnitVec{a}_{j,2}} \mtSkewSym{\mtVec{v}_{j,k}} \mtSkewSym{\mtVec{w}_{j}}.
\end{align}

All of the modified Jacobian matrix rows for the point feature $j$ observations at both keyframes are then compiled into a single block,
\begin{align}
\mtMat{K}_{j} &= 
\begin{bmatrix}
\mtMat{K}_{j}^{i(1),2} \\
\vdots \\
\mtMat{K}_{j}^{i(n_o),2} \\
\mtMat{K}_{j}^{h,1}
\end{bmatrix}
\end{align}
which maintains the same rank as the original Jacobian block for the point feature,
\begin{align}
\mtMat{J}_j = \mtEvaluated{ \mtPartial{\mtVec{g}_j(\mtVec{x})}{\mtVec{x}} }{\mtVec{x} = \mtEstVec{x}},
\end{align}
since the manipulations are performed by full rank elementary row and column operations matrices. The resulting matrix block has the structure shown in Figure~\ref{fig:Kj}.
\begin{figure}[!ht]
\centering
\includegraphics[width=0.48\textwidth]{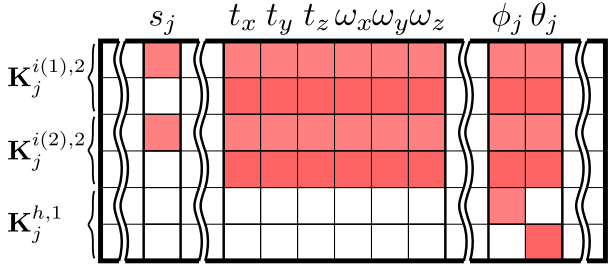}
\caption{Structure of the manipulated Jacobian block for point feature $j$ stacking all of the observations at both keyframes.}
\label{fig:Kj}
\end{figure}

It can be shown that the odd-numbered rows of $\mtMat{K}_j$ may always be written as a linear combination of the resulting even-numbered rows. Additionally, elementary row operations can eliminate all but the last elements in each of the columns associated with the point feature bearing states, $\phi_j$ and $\theta_j$. Finally, the non-zero element in the $s_j$ column can be used to eliminate the remaining non-zero elements in the first row. Subsequently, this new matrix, $\mtMat{L}_j$ has the same rank as the original Jacobian block $\mtMat{J}_j$ for this point feature and the structure is shown in Figure~\ref{fig:Lj}.
\begin{figure}[!ht]
\centering
\includegraphics[width=0.48\textwidth]{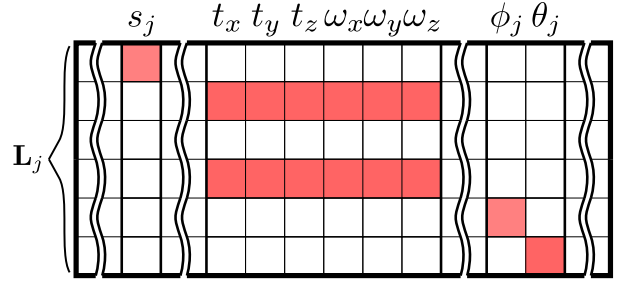}
\caption{Structure of the Jacobian block $\mtMat{L}_j$ for the point feature $j$ observations, resulting from manipulations to the original Jacobian.}
\label{fig:Lj}
\end{figure}

The matrix $\mtMat{M}$ is formed by stacking and reordering all of the $\mtMat{L}_j$ matrices for $j=1 \dots n_f$ and has the same rank as the original Jacobian $\mtMat{J}$. The structure of matrix $\mtMat{M}$ is shown in Figure~\ref{fig:M}.
\begin{figure}[!ht]
\centering
\includegraphics[width=0.48\textwidth]{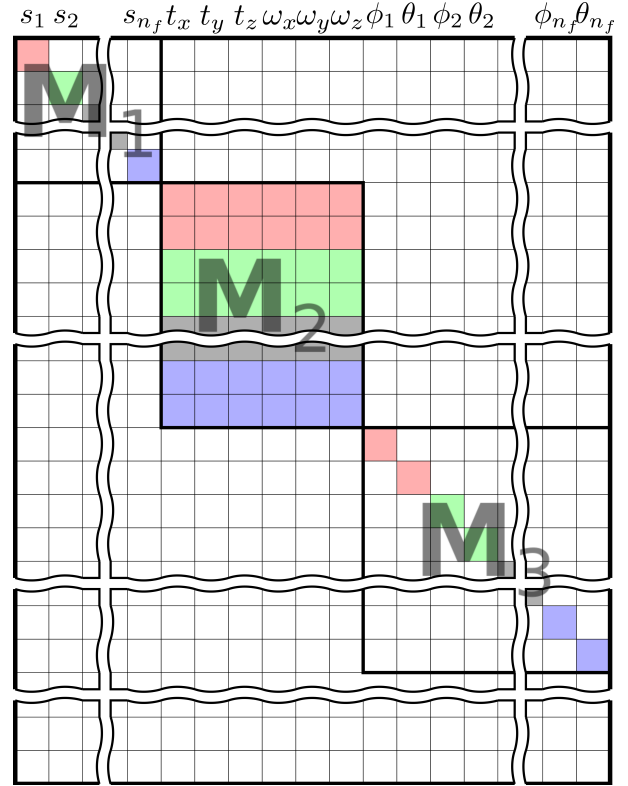}
\caption{Structure of the matrix $\mtMat{M}$, resulting from stacking and reordering the rows of the matrices $\mtMat{L}_j$. The sub-blocks of $\mtMat{M}$ are shown and all must be full rank for $\mtMat{M}$, and therefore $\mtMat{J}$, to be full rank.}
\label{fig:M}
\end{figure}

The matrix $\mtMat{M}$ is a block-diagonal matrix composed of three sub-matrices:
\begin{itemize}
\item $\mtMat{M}_1 \in \mtReal{\mtBy{n_f}{n_f}}$ is diagonal,
\item $\mtMat{M}_2 \in \mtReal{\mtBy{m_o}{6}}$,
\item $\mtMat{M}_3 \in \mtReal{\mtBy{2n_f}{2n_f}}$ is diagonal.
\end{itemize}
As a result, $\mtMat{M}$ and $\mtJacobian$ are full rank if and only if all of the following are satisfied,
\begin{itemize}
\item $\mtRank{\mtMat{M}_1} = n_f$,
\item $\mtRank{\mtMat{M}_2} = 6$, and
\item $\mtRank{\mtMat{M}_3} = 2n_f$.
\end{itemize}
It is clear that both $\mtMat{M}_1$ and $\mtMat{M}_3$ are full rank by construction, and therefore, $\mtMat{M}$ and by extension $\mtJacobian$ are full rank if and only if $\mtMat{M}_2$ is full rank, which concludes the proof.
\end{proof}

Therefore, determining if the measurement Jacobian $\mtJacobian$ is full rank is simplified to checking the rank of the reduced-dimension matrix $\mtMat{M}_2$. It follows directly from Lemma \ref{lemma:1}, that the degeneracy of the multicamera cluster SLAM system can be determined by checking the rank of $\mtMat{M}_2$.
\begin{mtCorollary}
For a multicamera cluster SLAM system satisfying Assumptions \ref{assum:1}, \ref{assum:2}, and \ref{assum:3}, the solution is degenerate and under-constrained, if and only if $\mtRank{\mtMat{M}_2} < 6$.
\end{mtCorollary}

\subsubsection{Rank of $\mtMat{M}_2$}
The matrix $\mtMat{M}_2$ is a dense $\mtBy{m_o}{6}$ block with a single row for each observation of the point features at the second keyframe. Each row of $\mtMat{M}_2$ in \eqref{eq:m2} specifies the six Pl\"{u}cker coordinates for a line in $\mtReal{3}$ since each set of coordinates satisfies the Grassmann-Pl\"{u}cker relation \cite{White:1994:grassmann},
\begin{align}
&\mtDot{\left( \mtSkewSym{\mtVec{v}_{j,k}} \mtUnitVec{a}_{j,2} \right)}{\left(  \mtSkewSym{ \mtVec{w}_j } \mtSkewSym{\mtVec{v}_{j,k}} \mtUnitVec{a}_{j,2}  \right)} \\
&= \mtDot{ - \mtVec{w}_j }{ \underbrace{ \left( \mtSkewSym{\mtSkewSym{\mtVec{v}_{j,i(k)}} \mtUnitVec{a}_{j,2}} \left( \mtSkewSym{\mtVec{v}_{j,i(k)}} \mtUnitVec{a}_{j,2} \right) \right)}_{= \ \mtZeros{3}{1}} } \\
&= 0. 
\end{align}
The matrix $\mtMat{M}_2$ will not have full rank when the $m_o$ sets of coordinates are linearly-dependent. This is similar to the problem of identifying motion singularities for series-parallel mechanisms. However, the current problem is more complex since the common connection points which sometimes allow for the simplification of the singularity condition in mechanisms are not present in the cluster SLAM system.

\subsection{Sufficient Conditions for Degeneracy}
\label{section:sufficient}
In this section, the structure of the matrix $\mtMat{M}_2$ from \eqref{eq:m2} will be exploited to identify cluster configurations and motions that are sufficient for degeneracy of the solution, independent of the number of point features observed and their constellation geometry.

It is immediately apparent that the system must include six point feature observations at the second keyframe for the matrix $\mtMat{M}_2$ to possibly have full rank. The first three columns in $\mtMat{M}_2$ are a stack of cross products involving the camera observation vectors and the bearings to the point features. When they all have a common collinear vector operand, the resulting row vectors are all coplanar, with the normal defined by the collinear vector operand. 

As expected, the system will be degenerate if the cluster consists of only one component camera since the rows will all have a common camera observation vector, consistent with how monocular vision systems are unable to recover the six degrees of freedom motion solution in a SLAM system. Additionally, the system is degenerate when only one point feature is observed by six cameras at the second keyframe since all of the matrix rows will contain the common point feature unit vector, $\mtUnitVec{a}_{1,2}$, in the cross product term. 

When the camera observation vectors are all parallel, the SLAM solution is degenerate. Each camera observation vector can be written as a scalar multiple, $\exists \gamma_{m,n} \in \mtReal{}$ such that $\mtVec{v}_{m,n} = \gamma_{m,n} \mtVec{v} \in \mtReal{3}$. In this case, the matrix $\mtMat{M}_2$ will have less than full rank since,
\begin{align}
\mtRank{
\begin{bmatrix}
-\gamma_{1,1} \mtTranspose{\mtUnitVec{a}_{1,2}} \mtSkewSym{\mtVec{v}} \\
\vdots \\
-\gamma_{n_f,n_o(n_f)} \mtTranspose{\mtUnitVec{a}_{n_f,2}} \mtSkewSym{\mtVec{v}}
\end{bmatrix}} \leq 2 < 3,
\end{align}
and $\mtRank{\mtMat{M}_2} < 6$. This condition includes the previously known two-camera cluster concentric circle degeneracies, since the motion causes the camera centres to move in parallel. Adding more cameras to the cluster reduces the configurations for which the relative motion will lead to the camera observation vectors being parallel. Additionally, when point features are observed across different cameras within the cluster at the two keyframes, it becomes less likely that all of the camera observation vectors are parallel. However, there do exist certain combinations of cluster motions for which the camera observation vectors remain parallel regardless of the feature point locations, and the camera cluster system becomes degenerate. Some example configurations are presented in Figure~\ref{fig:deg-ex}.

% Figure -- degenerate configurations example
\begin{figure*}[!ht]
\centering
\includegraphics[width=\textwidth]{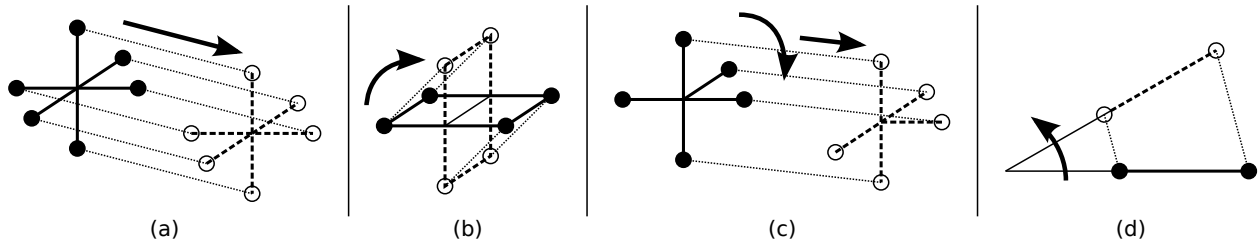}
\caption{Examples of camera cluster motions sufficient for degeneracy. The black dots are the cameras at the first keyframe connected by solid lines, white dots are the cameras at the second keyframe connected by dashed lines. The dotted lines are the camera observation vectors which are all parallel. Motions include (a) pure translation with no intercamera correspondence, (b) rotation axis in the plane of planar four-camera cluster, (c) 90 degrees rotation with translation, (d) two-camera concentric circles motion.}
\label{fig:deg-ex}
\end{figure*}

When the relative motion of the camera cluster is such that a point feature which was observed in one camera at the first keyframe is observed by a different camera at the second keyframe, there is a camera observation vector between the positions of the two cameras when they observed the particular point feature. This can create a set of camera observation vectors which are non-parallel even when the relative motion is a pure translation. As a result, it is possible that the system is non-degenerate and the solution can be found in this case, depending on the rank of the matrix $\mtMat{M}_2$. Observing common point features over multiple cameras is an effective way of avoiding the set of camera observation vectors becoming parallel and reducing the set of sufficient motions for system degeneracy.

\subsection{Necessary and Sufficient Conditions for Degeneracy}

In the case when the motion does not produce parallel camera observation vectors, it is necessary to evaluate the rank of the matrix $\mtMat{M}_2$ before concluding that the system is non-degenerate. The matrix $\mtMat{M}_2$ can be regarded as a set of motion constraints on a mechanism where a non-full rank means that the framework is not rigid and the configuration can change without violating the constraints. The full analysis of these singular configurations is beyond the scope of this work, but this section presents a set of example configurations for some general case systems to numerically demonstrate the effect of adding additional cameras and point feature observations on the degenerate configuration set.

Figures \ref{fig:cam2points6} and \ref{fig:cam3points6} show typical surface meshes for the relative cluster translation, $\mttK$, leading to $\mtMat{M}_2$ losing rank for example two and three-camera cluster systems observing six point features with no overlap in the camera observations and non-zero relative rotation. The surface is computed numerically as the locus of zero determinant for the matrix $\mtMat{M}_2$ using a set of randomly-positioned point features. A similar surface is generated for any non-zero rotation $\mtRK$.

\begin{figure}[t!]
\centering
\begin{subfigure}{0.48\textwidth}
  \centering
  \includegraphics[width=\textwidth]{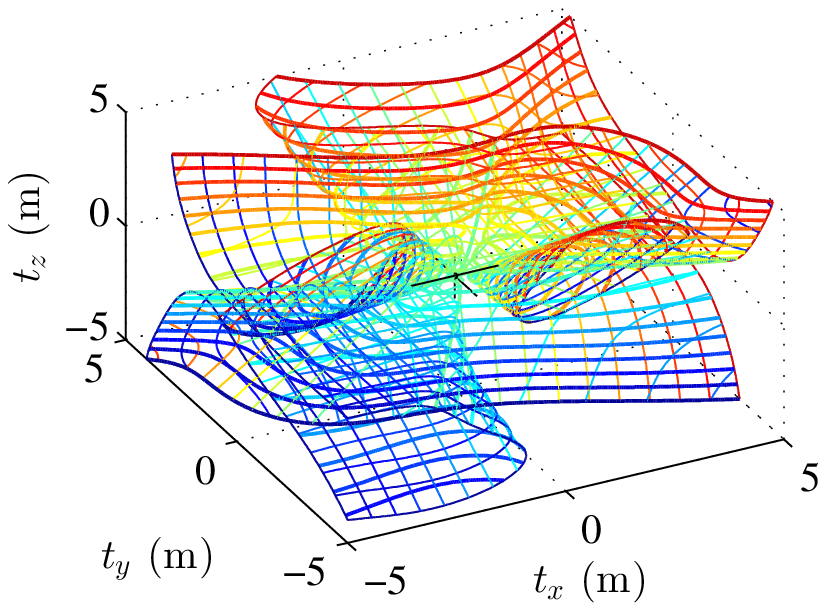}
  \caption{Two cameras} 
  \label{fig:cam2points6}
\end{subfigure}
\begin{subfigure}{0.45\textwidth}
  \centering
  \includegraphics[width=\textwidth]{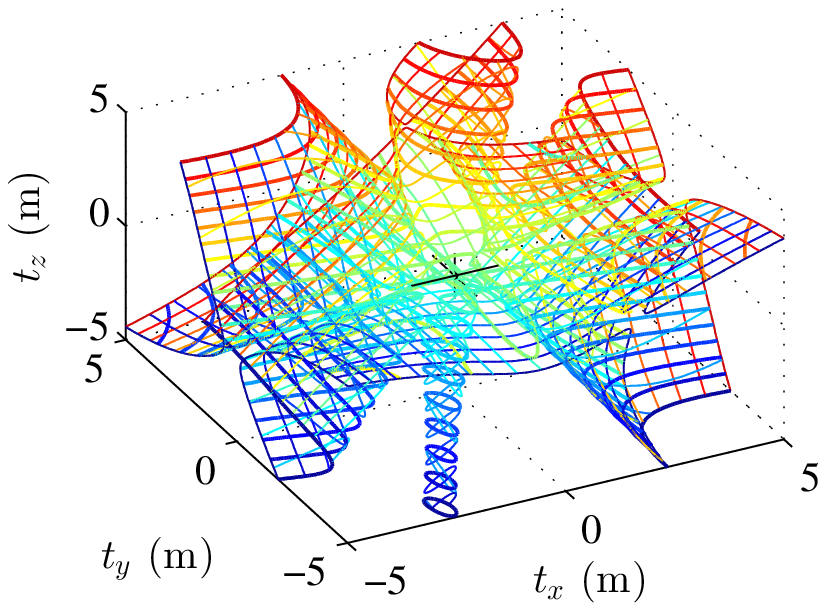} 
  \caption{Three cameras} 
  \label{fig:cam3points6}
\end{subfigure}
\caption{Degenerate $\mttK$ for (a) two and (b) three-camera clusters, observing six points with no overlap and non-zero rotation.}
\label{fig:campoints6}
\end{figure}

As more point feature observations are added to the system, the number of degeneracies is reduced. With more than six point feature observations, the size of the matrix $\mtMat{M}_2$ becomes $\mtBy{m_o}{6}$, where $m_o > 6$ and therefore the rank of the matrix cannot be checked by computing the determinant directly. Instead, for the matrix to not be full rank, all of the $m_o$ choose 6 submatrices of size $\mtBy{6}{6}$ formed by the rows of $\mtMat{M}_2$ must have a zero determinant. Each of the submatrices generates a surface as in Figure \ref{fig:campoints6}, and therefore, $\mtMat{M}_2$ has deficient rank at the $\mttK$ where all of the surfaces intersect. 

A two-dimensional cross-section at $\mtVar{t}_z = 1 \text{ m}$ is shown in Figure \ref{fig:cam3points78} for the three-camera case observing 7 and 8 point features with no overlap and non-zero rotation. The degenerate cluster translations correspond to the points on the graph where all of the curves intersect and are marked as black circles. These intersections are determined numerically using the computed loci for the determinants of the submatrices. If all of the loci intersect with each other within a certain epsilon ball, the location is selected as a degeneracy of $\mtMat{M}_2$.

\begin{figure}
\centering
\begin{subfigure}{0.45\textwidth}
  \centering
  \includegraphics[width=\textwidth]{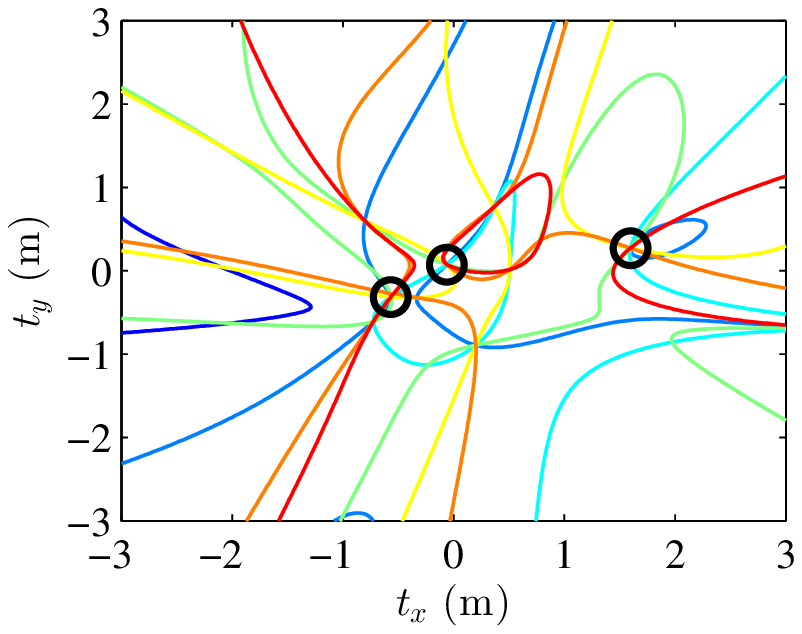} 
  \caption{7 points}
  \label{fig:cam3points7}
\end{subfigure}
\begin{subfigure}{0.45\textwidth}
  \centering
  \includegraphics[width=\textwidth]{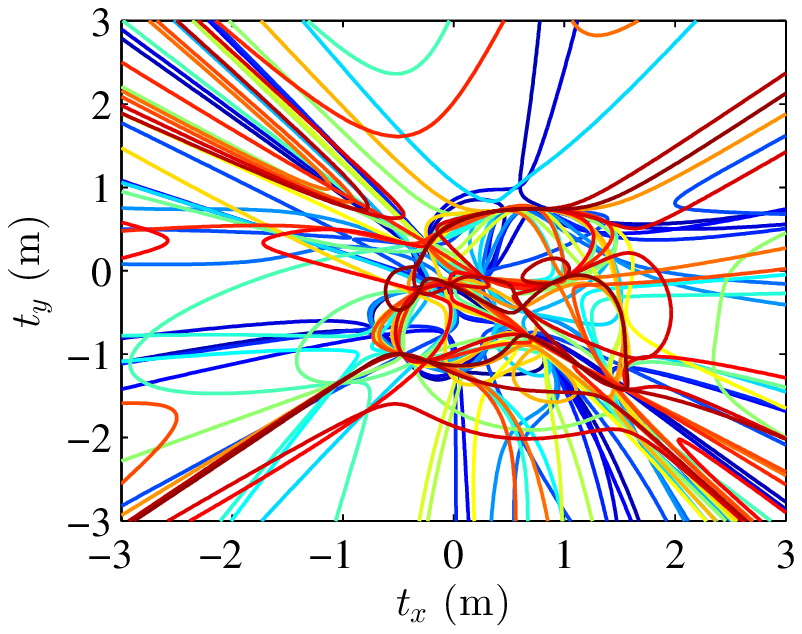} 
  \caption{8 points}
  \label{fig:cam3points8}
\end{subfigure}
\caption{The loci of zero determinants for the $m_o$ choose 6 submatrices of $\mtMat{M}_2$ at $\mtVar{t}_z = 1 \text{ m}$ in the three-camera cluster case for (a) $m_o=7$ and (b) $m_o=8$ point features. The degenerate points are the intersections marked with black circles.} 
\label{fig:cam3points78}  
\end{figure}

Notice that the number of degenerate motions is reduced from the curves of each colour for a subset of six feature observations, to a small finite set of points at any given cross-section. While the indicated degenerate positions are subject to numerical precision considerations, these examples are more informative in demonstrating that the system is non-degenerate in almost all configurations.

It is observed that the degenerate points in Figure \ref{fig:cam3points7} connect as lines in $\mtReal{3}$ at different slices of $\mtVar{t}_z$. When observing eight points in general position, Figure \ref{fig:cam3points8} shows that there are no translations for which the system is degenerate. While not exhaustive, numerical analysis of the singular configurations of $\mtMat{M}_2$ shows that the set of degenerate motions in the cluster system has been reduced from the previous surface with $m_o=6$, to a set of lines with $m_o=7$ and the empty set for $m_o=8$.

Adding point feature observations to the system is an effective way to reduce the set of degenerate motions for the camera cluster system; however, the sufficient conditions in Section \ref{section:sufficient} remain no matter how many point features are observed. Examples of these degeneracies are shown in Figure \ref{fig:cam2points8} for a two-camera cluster and Figure \ref{fig:cam3points8_inplane} for a three-camera cluster with rotation purely in the camera centre plane \cite{Tribou:2013:srmcsnfov}.

\begin{figure}
\centering
\begin{subfigure}{0.45\textwidth}
  \centering
  \includegraphics[width=\textwidth]{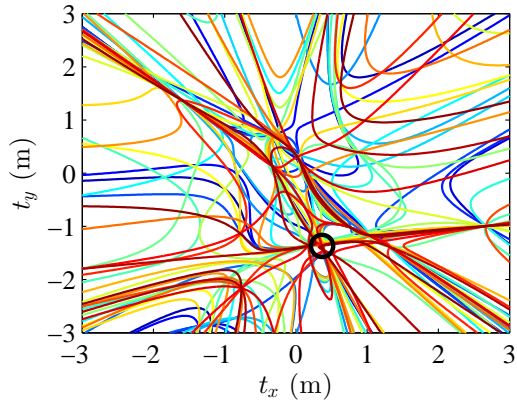} 
  \caption{Two cameras}
  \label{fig:cam2points8}
\end{subfigure}
\begin{subfigure}{0.45\textwidth}
  \centering
  \includegraphics[width=\textwidth]{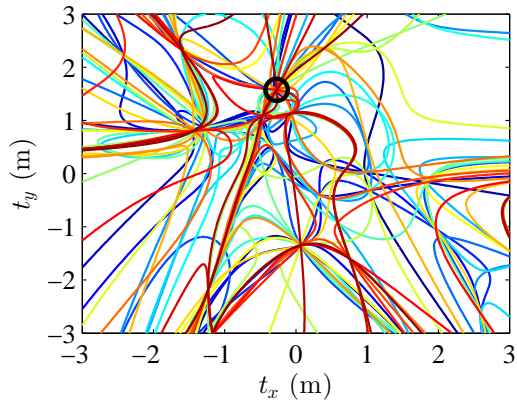} 
  \caption{Three cameras}
  \label{fig:cam3points8_inplane}
\end{subfigure}
\caption{The loci of zero determinants for the 8 choose 6 submatrices of $\mtMat{M}_2$ at $\mtVar{t}_z = 1 \text{ m}$ in the (a) two-camera and (b) three-camera cluster case with rotation axis within the camera centre plane observing eight point features} 
\label{fig:cam23points8}  
\end{figure}

The indicated degeneracy corresponds to the motion causing the camera observation vectors to move in parallel and extends along a line in $\mtReal{3}$. In order to eliminate these motions, it is necessary to add more cameras to the cluster, observe point features over different cameras, or both, to ensure that not all of the camera observation vectors are parallel.

Multiple cameras observing the same point feature at the different keyframes adds extra camera observation vectors which are less likely to be parallel with the rest of the vector set and produce a full-rank $\mtMat{M}_2$. An example two-camera cluster observing eight point features is shown in Figure \ref{fig:cam2points8_cross}. Camera 2 observes a feature at the second keyframe that was measured by camera 1 initially, and camera 1 observes a feature from camera 2. Significantly, there is no relative rotation between the keyframes of this system, but the rank of $\mtMat{M}_2$ is full nearly everywhere. This is an improvement on the previous completely non-overlapping cluster configurations which required non-zero rotation to have a full-rank $\mtMat{M}_2$.

\begin{figure}
  \centering
  \includegraphics[width=0.45\textwidth]{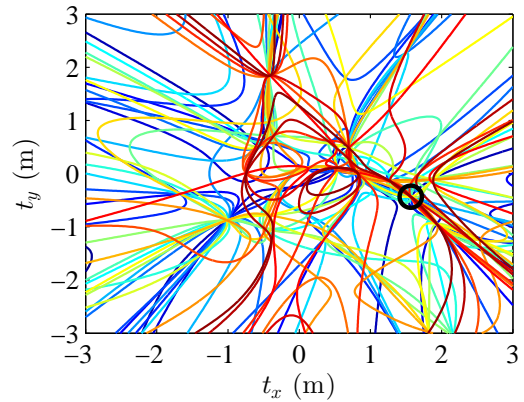} 
  \caption{The loci of zero determinants for the 8 choose 6 submatrices of $\mtMat{M}_2$ at $\mtVar{t}_z = 1 \text{ m}$ in the two-camera cluster system with zero rotation between keyframes, but two common features across cameras.}
  \label{fig:cam2points8_cross}
\end{figure}

The example systems in this section demonstrate numerically the effect of adding cameras and point feature observations to the camera cluster SLAM solution. For any algorithm, the overall strategy should be to reduce the number of degenerate configurations by increasing the number of point features observed on the target, and then eliminating the remaining sufficient conditions for degeneracy by adding cameras to the cluster, or observing point features across cameras, such that it is impossible for all of the camera observation vectors to be parallel through the motion. This will ensure a well-constrained solution to the localization and mapping problem.

% EXPERIMENTS
%\input{tex/experiments/experiments}

% CONCLUSIONS
\section{Conclusions}
\label{section:conclusions}

This work presented a detailed analysis of the degenerate configurations of the calibrated non-overlapping FOV multicamera cluster SLAM problem for an optimization based on minimizing a least-squares cost function with respect to the image-plane reprojection error. The system is reduced to a simple matrix rank test on a matrix consisting of rows of Pl\"{u}cker coordinates for lines in $\mtReal{3}$. Sufficient configurations for solution degeneracy caused by the relative motion were identified for $n_c$-camera clusters observing any number of point features over two keyframes. This leads to the novel general conclusion that if all of the camera observation vectors, formed as the displacement between the pairs of cameras observing a particular point feature, are parallel in a common coordinate frame, then the system is degenerate. It is further shown for several example systems that with the addition of more cameras to the cluster, more point feature observations, and observations of the point features across different cameras, the set of degenerate configurations is significantly reduced as it becomes impossible for all of the identified vectors to be parallel and the redundant observations prevent all of the determinants of the submatrices from going to zero concurrently.

Future work will focus on fully characterizing the necessary and sufficient conditions for the system to become degenerate, including the degeneracies related to the geometry of the point feature constellation from the standpoint of geometric algebra techniques \cite{Dorst:2010:gacsooag}. Additionally, the results from this work will be used to generate metrics for deciding when and where to add keyframes in a real-time SLAM system to accurately construct and constrain the generated target model and avoid the degeneracies within the state space caused by measurements from this type of sensor.

% ACKNOWLEDGEMENTS
\section*{Acknowledgements}

This work was partially funded by the National Sciences and Engineering Research Council of Canada (NSERC) under Grant No.~CRDPJ 397768-10. Partial funding also comes from the NSERC through the Alexander Graham Bell Canada Graduate Scholarship - Doctoral (CGS-D) award.

%% The Appendices part is started with the command \appendix;
%% appendix sections are then done as normal sections
%% \appendix

%% \section{}
%% \label{}
\appendix

\section{Projective Geometry}
\label{section:projective}
The projective space, $\mtProjective{n}$, consists of the real vector space $\mtReal{n}$, with the addition of points at infinity \cite{Hartley:2003:Mvgcv}. Only a very brief description of the projective space is presented here and the reader is referred to \cite{Hartley:2003:Mvgcv} for a more thorough introduction.

A point in the projective space is represented by the $n+1$ homogeneous coordinates,
\begin{align}
\mtHVec{x} = \mtTranspose{ \begin{bmatrix} \mtHVar{x}_1 & \mtHVar{x}_2 & \dots & \mtHVar{x}_{n+1} \end{bmatrix} } \in \mtProjective{n}.
\end{align}
The points at infinity in $\mtReal{n}$ are represented by those with coordinate $x_{n+1} = 0$. For finite points in $\mtReal{n}$ -- when $x_{n+1} \neq 0$ -- the coordinates of the corresponding point $\mtVec{x} \in \mtReal{n}$ are determined by,
\begin{align}
\mtVec{x} &= \mtTranspose{ \begin{bmatrix} x_1 & x_2 & \dots & x_{n} \end{bmatrix} } \\
&= \mtTranspose{ \begin{bmatrix} \dfrac{\mtHVar{x}_1}{\mtHVar{x}_{n+1}} & \dfrac{\mtHVar{x}_2}{\mtHVar{x}_{n+1}} & \dots & \dfrac{\mtHVar{x}_n}{\mtHVar{x}_{n+1}} \end{bmatrix} }.
\end{align}
Note that there is no way of mapping a point at infinity back to $\mtReal{n}$ since it would require division by zero. 

Each point along a ray in the projective space maps to the same point in the real vector space. As a result, the points $\mtHVec{x}$ and $\lambda \mtHVec{x}$, for $\lambda \in \mtReal{}$, map to the same point $\mtVec{x} \in \mtReal{n}$. Not surprisingly, there is an extra degree of freedom in the projective vectors using $n+1$ coordinates to represent a $n$-dimensional space. Finally, it is possible to represent any point $\mtVec{x} \in \mtReal{n}$ in the corresponding projective space $\mtProjective{n}$ simply by augmenting the coordinates,
\begin{align}
\mtHVec{x} = \mtTranspose{ \begin{bmatrix} \mtTranspose{ \mtVec{x} } & 1 \end{bmatrix} }.
\end{align}

The projective spaces allow for projective and coordinate transformations to be represented as linear matrix operations.

%% References
%%
%% Following citation commands can be used in the body text:
%% Usage of \cite is as follows:
%%   \cite{key}         ==>>  [#]
%%   \cite[chap. 2]{key} ==>> [#, chap. 2]
%%

%% References with bibTeX database:

\section*{References}
\bibliographystyle{elsarticle-num}
\bibliography{references}

\begin{thebibliography}{10}
\expandafter\ifx\csname url\endcsname\relax
  \def\url#1{\texttt{#1}}\fi
\expandafter\ifx\csname urlprefix\endcsname\relax\def\urlprefix{URL }\fi
\expandafter\ifx\csname href\endcsname\relax
  \def\href#1#2{#2} \def\path#1{#1}\fi

\bibitem{Baker:2001:semcmbmw}
P.~Baker, C.~Fermuller, Y.~Aloimonos, R.~Pless, A spherical eye from multiple
  cameras (makes better models of the world), in: Proceedings of the IEEE
  Conference on Computer Vision and Pattern Recognition (CVPR), 2001, pp.
  576--583.

\bibitem{Pless:2003:Umcao}
R.~Pless, {Using many cameras as one}, in: Proceedings of the IEEE Conference
  on Computer Vision and Pattern Recognition (CVPR), Vol.~2, 2003, pp.
  II--587--593.

\bibitem{Clipp:2008:RDMENMS}
B.~Clipp, J.~H. Kim, J.~M. Frahm, M.~Pollefeys, R.~Hartley, Robust 6{DOF}
  motion estimation for non-overlapping, multi-camera systems, in: Proceedings
  of the IEEE Workshop on Applications of Computer Vision (WACV), 2008, pp.
  1--8.

\bibitem{Thrun:2005:PR}
S.~Thrun, W.~Burgard, D.~Fox, Probabilistic robotics, The MIT Press, 2005.

\bibitem{Hartley:2003:Mvgcv}
R.~Hartley, A.~Zisserman, {Multiple view geometry in computer vision},
  Cambridge University Press, 2003.

\bibitem{Davison:2007:MRSCS}
A.~J. Davison, I.~D. Reid, N.~D. Molton, O.~Stasse, {MonoSLAM}: {R}eal-time
  single camera {SLAM}, IEEE Transactions on Pattern Analysis and Machine
  Intelligence 29~(6) (2007) 1052--1067.

\bibitem{Hermann:1977:Nco}
R.~Hermann, A.~Krener, {Nonlinear controllability and observability}, IEEE
  Transactions on Automatic Control 22~(5) (1977) 728--740.

\bibitem{Kim:2010:Remsmtc}
J.~H. Kim, M.~J. Chung, B.~T. Choi, Recursive estimation of motion and a scene
  model with a two-camera system of divergent view, Pattern Recognition 43~(6)
  (2010) 2265--2280.

\bibitem{Tribou:2013:srmcsnfov}
M.~J. Tribou, S.~L. Waslander, D.~W.~L. Wang, Scale recovery in multicamera
  cluster {SLAM} with non-overlapping fields of view, submitted to
  \textit{Computer Vision and Image Understanding} (Aug. 2013).

\bibitem{Stewenius:2005:StMGRPP}
H.~Stewenius, D.~Nister, M.~Oskarsson, K.~Astrom, Solutions to minimal
  generalized relative pose problems, in: Proceedings of the Workshop on
  Omnidirectional Vision, 2005.

\bibitem{Sturm:2005:MGGCM}
P.~Sturm, Multi-view geometry for general camera models, in: Proceedings of the
  IEEE Conference on Computer Vision and Pattern Recognition (CVPR), Vol.~1,
  2005, pp. 206--212.

\bibitem{Mouragnon:2009:Grsmulba}
E.~Mouragnon, M.~Lhuillier, M.~Dhome, F.~Dekeyser, P.~Sayd, Generic and
  real-time structure from motion using local bundle adjustment, Image and
  Vision Computing 27~(8) (2009) 1178--1193.

\bibitem{Kim:2010:DotLSPAfGE}
J.~S. Kim, T.~Kanade, Degeneracy of the linear seventeen-point algorithm for
  generalized essential matrix, Journal of Mathematical Imaging and Vision
  37~(1) (2010) 40--48.

\bibitem{Li:2008:latmeugcm}
H.~Li, R.~Hartley, J.~H. Kim, A linear approach to motion estimation using
  generalized camera models, in: Proceedings of the IEEE Conference on Computer
  Vision and Pattern Recognition (CVPR), 2008, pp. 1--8.

\bibitem{Lu:2004:AsmDra}
Y.~Lu, J.~Z. Zhang, Q.~M.~J. Wu, Z.~N. Li, {A survey of motion-parallax-based
  3-D reconstruction algorithms}, IEEE Transactions on Systems, Man, and
  Cybernetics 34~(4) (2004) 532--548.

\bibitem{Murray:1994:amitrm}
R.~M. Murray, Z.~Li, S.~S. Sastry, A mathematical introduction to robotic
  manipulation, CRC Press, 1994.

\bibitem{Rosten:2005:Fplhpt}
E.~Rosten, T.~Drummond, {Fusing points and lines for high performance
  tracking.}, in: Proceedings of the IEEE International Conference on Computer
  Vision (ICCV), Vol.~2, 2005, pp. 1508--1511.

\bibitem{Rosten:2006:Mlhscd}
E.~Rosten, T.~Drummond, Machine learning for high-speed corner detection, in:
  Proceedings of the European Conference on Computer Vision (ECCV), Vol.~1,
  2006, pp. 430--443.

\bibitem{Lowe:1999:ORLSF}
D.~G. Lowe, Object recognition from local scale-invariant features, Proceedings
  of the IEEE International Conference on Computer Vision (ICCV) 2 (1999)
  1150--1157.

\bibitem{Bay:2006:SSurf}
H.~Bay, A.~Ess, T.~Tuytelaars, L.~{Van Gool}, Speeded-up robust features
  ({SURF}), Computer Vision and Image Understanding 110~(3) (2008) 346--359.

\bibitem{Klein:2007:PTMSAW}
G.~Klein, D.~Murray, Parallel tracking and mapping for small {AR} workspaces,
  in: Proceedings of the IEEE and ACM International Symposium on Mixed and
  Augmented Reality (ISMAR), 2007, pp. 225--234.

\bibitem{White:1994:grassmann}
N.~L. White, {Grassmann-Cayley} algebra and robotics, Journal of Intelligent
  and Robotic Systems 11~(1-2) (1994) 91--107.

\bibitem{Dorst:2010:gacsooag}
L.~Dorst, D.~Fontijne, S.~Mann, Geometric Algebra for Computer Science: An
  Object-Oriented Approach to Geometry, Morgan Kaufmann, 2010.

\end{thebibliography}

%% Author bios
%\input{tex/biography}

\end{document}